\documentclass[12pt,a4paper]{article}
\usepackage{amssymb}
\usepackage{amsmath}
\usepackage[symbol]{footmisc}
\usepackage{amsthm}
\usepackage{amscd}
\usepackage{multirow}
\usepackage{amstext}
\usepackage{booktabs}
\usepackage{vmargin}
\usepackage{float}
\usepackage{fancyhdr}
\setpapersize{A4}
\setmarginsrb{3.0cm}{1.5cm}{3.0cm}{3cm}{1cm}{.5cm}{2cm}{2cm}

\usepackage{amsfonts}
\usepackage[numbers]{natbib}
\usepackage[latin1]{inputenc}
\usepackage[T1]{fontenc}
\usepackage[american]{babel}
\usepackage{graphicx}
\usepackage{bbm}
\usepackage[mathscr]{euscript}
\usepackage{mathrsfs}
\usepackage{stmaryrd}
\usepackage{soul}
\usepackage{hyperref}
\usepackage{xspace}
\usepackage[draft,danish]{fixme}
\usepackage{mathtools}
\usepackage{dsfont}
\usepackage{url}
\usepackage[ps,all,dvips]{xy}
\usepackage[shortlabels,inline]{enumitem}
\setlist[enumerate]{font=\textnormal}

\usepackage{caption}

\SelectTips{cm}{12}
\CompileMatrices

  \makeatletter
 \useshorthands{"}
 \defineshorthand{"-}{\nobreak-\bbl@allowhyphens}
 \makeatother

\newcommand{\dd}{\operatorname{d\mkern-2.5mu}}

\newcommand{\distEq}[1][]{%
	\overset{%
		\smash{%
			\raisebox{-0.5pt}{%
				\ $
				\scriptscriptstyle%
				\mathscr{D}%
				$
			}%
		}%
	}{%
		=%
	}
}

\newcommand{\pkg}[1]{{\fontseries{b}\selectfont #1}}

\DeclareMathOperator{\MSE}{MSE}	
\DeclareMathOperator{\Leb}{Leb}	
\DeclareMathOperator{\Var}{Var}
\DeclareMathOperator{\Cov}{Cov}
\DeclareMathOperator{\diam}{diam}
\DeclareMathOperator{\sign}{sign}

\theoremstyle{plain}
\newtheorem{lemma}{Lemma}[section]

\newtheorem{theorem}[lemma]{Theorem}
\newtheorem{corollary}[lemma]{Corollary}
\theoremstyle{definition}

\theoremstyle{definition}

\theoremstyle{remark}
\newtheorem{remark}[lemma]{Remark}

\usepackage{pgfplots}
\pgfplotsset{
	every tick label/.append style={font=\footnotesize}
}

\tikzset{
	invisible/.style={opacity=0},
	visible on/.style={alt={#1{}{invisible}}},
	alt/.code args={<#1>#2#3}{%
		\alt<#1>{\pgfkeysalso{#2}}{\pgfkeysalso{#3}}
	},
}

\pgfplotsset{compat=newest}

\makeatletter \newcommand{\pgfplotsdrawaxis}{\pgfplots@draw@axis} \makeatother

\pgfplotsset{axis line on top/.style={
		axis line style=transparent,
		ticklabel style=transparent,
		tick style=transparent,
		axis on top=false,
		after end axis/.append code={
			\pgfplotsset{axis line style=opaque,
				ticklabel style=opaque,
				tick style=opaque,
				grid=none}
			\pgfplotsdrawaxis}
	}
}

\newcommand{\devnull}[1]{}

\numberwithin{equation}{section}

\title{Modeling of time series using random forests: theoretical developments}
\author{Richard A.~Davis\footnote{Department of Statistics, Columbia University, USA. E-mail: rdavis@stat.columbia.edu.} \and Mikkel S.~Nielsen\footnote{Department of Statistics, Columbia University, USA. E-mail: m.nielsen@columbia.edu.}}
\date{ }
	
\begin{document}
\maketitle

\begin{abstract}
In this paper we study asymptotic properties of random forests within the framework of nonlinear time series modeling. While random forests have been successfully applied in various fields, the theoretical justification has not been considered for their use in a time series setting. Under mild conditions, we prove a uniform concentration inequality for regression trees built on nonlinear autoregressive processes and, subsequently, we use this result to prove consistency for a large class of random forests. The results are supported by various simulations.
%
\end{abstract}


\section{Introduction}
Random forests, originally introduced by \citet{breiman2001random}, constitute an ensemble learning algorithm for classification and regression, which produces predictions by first growing a large number of randomized decision trees \cite{breiman1984classification} and, then, aggregates the results. Since its introduction, the algorithm has been applied in various fields such as object recognition \cite{shotton2011real}, bioinformatics \cite{diaz2006gene}, ecology \cite{cutler2007random,prasad2006newer} and finance \cite{gu2020empirical,kumar2006forecasting}, and the evidence is strong: with very little tuning, random forests are able to deliver a flexible tool for prediction which is fully comparable with other state-of-the-art algorithms. In fact, \citet{howard2012two} claim that random forests have been the most successful general-purpose algorithm in recent times. While many successful applications indicate the wide applicability of random forests, only little theoretical work exists to support this impression. Among components, which make the random forests of \citet{breiman2001random} difficult to analyze, are the operation of bagging randomized predictors \cite{breiman1996bagging} as well as the highly data-dependent partitions associated to the so-called CART regression trees \cite{breiman1984classification}, which form the forest. Other types of random forests have been proposed; see, for example, \cite{amaratunga2008enriched,geurts2006extremely}.

While the bagging step is often discarded in theoretical work, or replaced by another resampling method such as subsampling, asymptotic results for random forests in the (classical) nonparametric regression setting, where $(X_1,Y_1),\dots, (X_T,Y_T)$ are i.i.d.\ observations from the model
\begin{equation}\label{regressionSetting}
Y = f(X) + \varepsilon,
\end{equation}
have been established under rather weak assumptions on the structure of the underlying regression trees. In \eqref{regressionSetting}, $f$ is a suitable smooth function and $\varepsilon$ is a mean-zero square integrable noise term which is independent of $X$. To mention a few significant results in this setup, \citet{consRF} prove $L^2$ consistency of Breiman's random forests when $f$ is additive (i.e., $f(x) = \sum_{i=1}^p f_i (x_i)$) and $\varepsilon$ is Gaussian, \citet{wager2015adaptive} establish pointwise consistency of similar forests with larger leaves in a high-dimensional setting, and \citet{wager2018estimation} prove (pointwise) asymptotic normality of a particular random forest algorithm. Although assumptions are more restrictive, valuable insights about performance (i.e., convergence rates) in sparse settings and lower bounds on mean squared error were provided by \cite{biau2012analysis,biau2010layered,lin2006random}. For a nice overview of existing theoretical work on random forests within the regression setting as well as further references, see the survey in \citet{biau2016random}.

In some applications, particularly financial, the underlying data correspond to observations from a time series, and the aim is to predict future values by feeding in a number of the most recent observations to the algorithm. While the problem is often treated precisely as in the regression setting from a practical point of view, by forming pairs $(X_1,Y_1),\dots, (X_T,Y_T)$ where $X_t = (Y_{t-1},\dots, Y_{t-p})$ for some integer $p\geq 1$, things change dramatically on the theoretical side. Indeed, observations can no longer be assumed to be i.i.d.\ draws from \eqref{regressionSetting} and, instead, the entire process $(Y_t)_{t\geq 1}$ is necessarily defined recursively by the equation
\begin{equation}\label{nlar}
Y_t = f(X_t) + \varepsilon_t,\qquad t \geq 1,
\end{equation}
given initial data $\xi = (Y_0,Y_{-1},\dots, Y_{1-p})$. Processes satisfying \eqref{nlar} are often referred to as nonlinear autoregressive processes of order $p$ (or, in short, NLAR($p$) processes). For further detail on these processes, see \cite{an1996geometrical,hardle1997review}. In such a framework, the dependence structure, across pairs $(X_1,Y_1),\dots, (X_T,Y_T)$ as well as between entries in $X_t$, is determined within the model. Consequently, in contrast to the regression setup, it is often only appropriate to impose assumptions on $f$ and $(\varepsilon_t)_{t\geq 1}$. In fact, even if one accepts an implicit model assumption, e.g., the typical assumption that $X_t$ admits a copula density which is bounded away from zero and infinity, it turns out to be rather restrictive. Indeed, if $(Y_t)_{t\geq 1}$ is Gaussian and $p\geq 2$, such an assumption is satisfied only if $f = 0$ almost everywhere. It follows that other types of assumptions and techniques are needed to guarantee the validity of random forests in the time series setting. 

In this paper we rely on the principal ideas of \cite{wager2015adaptive} to obtain a uniform concentration inequality which applies simultaneously across all regression trees satisfying a mild condition on their minimum leaf size $k$, when data are generated by the NLAR($p$) model \eqref{nlar}. While it is required that $k$ increases in the sample size, the growth rate may be very slow and trees are allowed to be grown adaptively (the partitions of the trees can be highly data-dependent). As an application of the established concentration inequality, we prove that all random forests respecting a number of conditions are pointwise consistent estimators for $f$ when the data generating process is \eqref{nlar}. The assumptions we impose in the model \eqref{nlar} are explicit in terms of $f$ and the distribution of $(\varepsilon_t)_{t\geq 1}$, and they are not difficult to check. For instance, all our results are applicable if $f$ is bounded and Lipschitz continuous, and $(\varepsilon_t)_{t\geq 1}$ is an i.i.d.\ sequence with $\varepsilon_1$ having a suitably light-tailed distribution. (As pointed out in Section~\ref{concentrationTrees}, the assumption of $f$ being bounded is no stronger than what is usually imposed in the regression setting.) Our techniques rely on, among other things, the theory of Markov processes as well as various Bernstein type concentration inequalities. To the best of our knowledge, theoretical properties of random forests within the framework of time series have not been fully addressed.

The paper is laid out as follows. Section~\ref{concentrationTrees} introduces the model as well as the regression trees of interest and establishes uniform concentration of these around their so-called partition-optimal counterparts (Theorem~\ref{puttingThingsTogether}). In Section~\ref{consistency} we translate this result into a concentration inequality for random forests (Corollary~\ref{RFconcentration}) and provide sufficient conditions ensuring that they are pointwise consistent estimators of $f$ (Theorem~\ref{RFconsistency}). Subsequently, we carry out a simulation study in Section~\ref{simulation} which considers the performance of random forests within the NLAR($p$) model for various specifications of $f$. Finally, Section~\ref{proofs} contains proofs of all statements as well as a number of auxiliary results.

\section{Concentration of regression trees around partition-optimal counterparts}\label{concentrationTrees}
Let $(\varepsilon_t)_{t\geq 1}$ be a sequence of i.i.d. random variables with $\mathbb{E}[\varepsilon_1] = 0$ and $\mathbb{E}[\varepsilon_1^2]<\infty$, and fix an integer $p \geq 1$. Given a vector $\xi = (Y_{0},Y_{-1},\dots, Y_{1-p})$ of initial data independent of $(\varepsilon_t)_{t\geq 1}$ and a measurable function $f\colon \mathbb{R}^p\to \mathbb{R}$, define the process $(Y_t)_{t\geq 1}$ recursively by 
\begin{equation}\label{mainObjective}
Y_t = f(X_t) + \varepsilon_t,\qquad t \geq 1,
\end{equation}
where $X_t \coloneqq (Y_{t-1},\dots, Y_{t-p})$. In addition to the initial data $\xi$, suppose that we have $T$ observations $Y_1,\dots, Y_T$ from the model \eqref{mainObjective} available and that we group them in input-output pairs, 
\begin{equation*}
\mathcal{D}_T = \{(X_1,Y_1),\dots, (X_T,Y_T)\}.
\end{equation*}
The aim of this section is to establish uniform concentration inequalities for regression trees built on $\mathcal{D}_T$. We start by recalling the associated concept of recursive partitions \cite{breiman1984classification}, which is used to construct regression trees. Define a sequence of partitions $\mathcal{P}_1,\mathcal{P}_2,\dots$ by starting from $\mathcal{P}_1 = \{\mathbb{R}^p\}$ and then, for each $n\geq 1$, construct $\mathcal{P}_{n+1}$ from $\mathcal{P}_n$ by replacing one set (node) $A\in \mathcal{P}_n$ by $A_L \coloneqq \{x\in A\, :\, x_i \leq \tau\}$ and $A_R\coloneqq \{x\in A\, :\, x_i > \tau\}$, where the split direction $i \in \{1,\dots, p\}$ and split position $\tau \in \{x_i\, :\, x \in A\}$ are chosen in accordance with some set of rules. Here $x_i$ refers to the $i$-th entry of $x \in \mathbb{R}^p$. In this context, we will say that $A$ is the \textit{parent} node of $A_L$ and $A_R$, while $A_L$ and $A_R$ are the \textit{child} nodes of $A$. A given partition $\Lambda$ of $\mathbb{R}^p$ is called recursive if $\Lambda = \mathcal{P}_n$ for some $n\geq 1$, where $\mathcal{P}_1,\dots, \mathcal{P}_n$ are obtained as above. Note that the rules determining how to choose node, direction and position of a split may depend on the data $\mathcal{D}_T$ as well as some injected randomness $\Theta$. For instance, in Breiman's random forests a node is split as soon as it contains at least a certain number of observations, while the position and direction are determined by maximizing impurity decrease (or, equivalently, minimizing the total mean-corrected sum of squares of the outputs $Y$ over the resulting two child nodes; see also \cite{breiman1984classification}), but only over a randomly chosen subset of directions in $\{1,\dots, p\}$. To any recursive partition $\Lambda$ we define the corresponding regression tree $T_\Lambda$ by
\begin{equation}\label{sampleTree}
T_\Lambda (x) = \frac{1}{\vert \{t\in \{1,\dots, T\}\, :\, X_t \in A_\Lambda (x)\}\vert }\sum_{t=1}^TY_t\mathds{1}_{A_\Lambda (x)}(X_t),\qquad x \in \mathbb{R}^p.
\end{equation}
Here the notation $A_\Lambda (x)$ is used to refer to the unique set $A\in \Lambda$ with the property that $x\in A$. Our interest will be on regression trees defined by $k$-valid partitions ($k\geq 1$). We will say that a partition $\Lambda$ is $k$-valid, and write $\Lambda \in \mathcal{V}_{k}$, if $\Lambda$ is recursive and each set in $\Lambda$ (sometimes called a leaf of the corresponding tree $T_\Lambda$) contains at least $k$ data points. Note that, since $\Lambda$ is recursive, it can depend on both the data $\mathcal{D}_T$ and a random mechanism $\Theta$, while $\mathcal{V}_{k}$ depends only on $\mathcal{D}_T$. Setting a minimum number $k$ of observations in each leaf of a tree is default in most practical implementations of random forests. Besides, such an assumption is natural since it ensures that $X_t\in A_\Lambda (x)$ for some $t\in \{1,\dots, T\}$, and this implies that the regression tree \eqref{sampleTree} is well-defined for all $x\in \mathbb{R}^p$. In this section we will be working under the following set of assumptions:
\begin{enumerate}[label=(A\arabic*)]
	\item\label{noiseDist}  The random variable $\varepsilon_1$ admits a density $h_\varepsilon$ which is positive almost everywhere on $\mathbb{R}$ and, for some $c\in (0,\infty)$,
	\begin{equation}\label{subGaussian}
	\mathbb{E}[\vert \varepsilon_1 \vert^m]\leq m!c^{m-2},\qquad m=3,4,\dots
	\end{equation}
	Moreover, the cumulative distribution function $F_\varepsilon (x) = \int_{-\infty}^x h_\varepsilon (y)\, \dd y$ of $\varepsilon_1$ satisfies
	\begin{equation}\label{tailCond}
	\sup_{x \in \mathbb{R}} \frac{F_\varepsilon (x+\tau)}{F_\varepsilon (x)}<\infty
	\end{equation}
	for any $\tau \in (0,\infty)$.
	
	\item\label{sparsity} The function $f$ in \eqref{mainObjective} is bounded,
	\begin{equation}\label{finiteRegression}
	M\coloneqq \sup_{x \in \mathbb{R}^{p}} \vert f(x)\vert <\infty.
	\end{equation}
	
	\item\label{leavesAssumption} The minimum number of leaves $k$ satisfies $k/(\log T)^4 \to \infty$ as $T\to \infty$.
	
\end{enumerate}

\noindent In contrast to $k$, the quantities $c$, $M$ and $p$ will be kept fixed, and hence we will not keep track of the dependence on these in the following results. In particular, the introduced constants can depend on $c$, $M$ and $p$, but not on $T$ and $k$. When $\varepsilon_1$ admits a density which is positive almost everywhere and $f$ satisfies \eqref{finiteRegression} (in particular, when \ref{noiseDist} and \ref{sparsity} are imposed), it follows by \cite[Theorem~3.1]{an1996geometrical} that the distribution of $\xi$ can be chosen such that $(Y_t)_{t\geq 1}$ is strictly stationary and, thus, this will be assumed throughout the paper. This means that $(X_1,Y_1),\dots, (X_T,Y_T)$ are identically distributed. Before turning to the results, let us attach some comments to the assumptions stated above. The assumption of \ref{noiseDist} that $\varepsilon_1$ has a positive density is convenient, since it ensures that the $p$-th order Markov chain $(Y_t)_{t\geq 1}$ can reach any state in one time step. In addition to strict stationarity of the chain, when combined with  \ref{sparsity}, the assumption ensures geometrical ergodicity as well. While it is not required that $f$ is bounded to prove such properties of $(Y_t)_{t\geq 1}$, we need boundedness to apply Bernstein type inequalities for weakly dependent processes and to obtain good estimates on the dependency between entries of the input vector $X_1$. The boundedness assumption \eqref{finiteRegression} is implicitly assumed in essentially all theoretical work on random forests as one usually assumes that the input vector is transformed so that it belongs to the unit cube $[0,1]^p$ and then requires continuity of $f$ on this domain. The assumption on the moments of $\varepsilon_1$ in \eqref{subGaussian} is known as Bernstein's condition and implies that $\varepsilon_1$ is sub-exponential in the sense that
\begin{equation}\label{subexponential}
\mathbb{P}(\vert \varepsilon_1 \vert >x) \leq \gamma_1 e^{-\gamma_2 x},\qquad x>0,
\end{equation}
for suitably chosen $\gamma_1,\gamma_2\in (0,\infty)$. It is a well-known assumption to impose when proving concentration inequalities and is often needed when $\varepsilon_1$ cannot be assumed bounded. Among distributions satisfying Bernstein's condition \eqref{subGaussian} are (sub-)Gaussian distributions, but also those with a slightly heavier tail such as the Laplace distribution. The assumption \eqref{tailCond} is used in conjunction with \eqref{finiteRegression} to estimate probabilities involving the input vector $X_1$ (see Lemma~\ref{copula} for details). Ultimately, it is an assumption on the left tail of $\varepsilon_1$, and a sufficient condition for this to hold is that the limit
\begin{equation*}
\lim_{x\to -\infty}\frac{h_\varepsilon (x)}{h_\varepsilon (x+\tau)}
\end{equation*}
exists and is non-zero for all $\tau \in (0,\infty)$. It is straightforward to verify that this, as well, is satisfied for both Gaussian and Laplace distributions. Together with \eqref{finiteRegression}, \eqref{tailCond} ensures that we do not need to impose conditions on the copula density of the input vector $X_1$, as is usually done in the regression setting, and this is convenient since such conditions can be both difficult to verify and even rather restrictive in a time series setting. Finally, we impose \ref{leavesAssumption}, which in particular implies that $k \to \infty$ as $T\to \infty$. Although it is allowed that $k\to \infty$ at a slow rate, the assumption contrasts the trees used in the random forests of \citet{breiman2001random}, where $k$ is some fixed and often small number. On the other hand, \ref{leavesAssumption} is very similar to assumptions imposed in most theoretical work within the regression setting (see, e.g., \cite{biau2012analysis,consRF,wager2015adaptive}). In fact, to the best of our knowledge, the only asymptotic result for random forests built on trees with fixed $k$ is \cite[Theorem~2]{consRF}. The logarithmic factor $(\log T)^4$ is related to the fact that the established bound applies uniformly across all trees (see Remark~\ref{comparison}) and that we use a Bernstein type inequality for strongly mixing processes which is slightly weaker than the classical one for the independent case.

While a couple of additional assumptions are needed to establish consistency of random forests in Section~\ref{consistency}, \ref{noiseDist}--\ref{leavesAssumption} are sufficient to prove that regression trees of the form \eqref{sampleTree} concentrate around their so-called partition-optimal counterparts
\begin{equation}\label{POtree}
T^\ast_\Lambda (x) \coloneqq \mathbb{E}_\Lambda [Y\mid X\in A_\Lambda(x)]
\end{equation}
uniformly across $(x,\Lambda)\in \mathbb{R}^p \times \mathcal{V}_k$. Here $(X,Y)$ is a copy of $(X_1,Y_1)$ which is independent of $(\mathcal{D}_T,\Theta)$, and $\mathbb{E}_\Lambda$ denotes expectation with respect to the conditional probability measure $\mathbb{P}_\Lambda \coloneqq \mathbb{P}(\: \cdot \mid \mathcal{D}_T ,\Theta)$. Conditional on $(\mathcal{D}_T,\Theta)$, the set $A_\Lambda (x)$ is non-random and, hence, the right-hand side of \eqref{POtree} simply means that the map $A\mapsto \mathbb{E}[Y\mid X\in A]$ is evaluated at $A_\Lambda(x)$. Our setting is very similar to that of \citet{wager2015adaptive}, but besides requiring partitions to be $k$-valid they impose an additional assumption that excludes too ``unbalanced'' splits (see also the trees constructed in Section~\ref{consistency}). 
\begin{theorem}\label{puttingThingsTogether}
	Suppose that \ref{noiseDist}--\ref{leavesAssumption} are satisfied. Then there exists a constant $\beta\in (0,\infty)$ such that
	\begin{equation}\label{mainInequality}
	\sup_{(x,\Lambda)\in \mathbb{R}^p\times \mathcal{V}_k} \vert T_\Lambda (x) - T^\ast_\Lambda (x) \vert \leq \beta \frac{(\log T)^2}{\sqrt{k}}
	\end{equation}
	with probability at least $1-4T^{-1}$ for all sufficiently large $T$.
\end{theorem}

\begin{remark}\label{comparison}
	For any given pair $(x,\Lambda)\in \mathbb{R}^p\times \mathcal{V}_{k}$, the quantity $\vert T_\Lambda (x) - T^\ast_\Lambda (x)\vert$ is the deviation of the sample average over at least $k$ observations from its theoretical counterpart within a specific leaf $L$. Some of the leaves, which can be obtained by varying $(x,\Lambda)$, contain only $k$ observations and for these, the error is of order $1/\sqrt{k}$. This is almost the same upper bound as in \eqref{mainInequality} apart from the logarithmic factor $(\log T)^2$, which reflects the fact that the deviation is controlled simultaneously across all feasible pairs $(x,\Lambda)$ as well as the sub-exponential tail of $\varepsilon_1$.
\end{remark}

\begin{remark}\label{autoregOrder}
	In Theorem~\ref{puttingThingsTogether}, and the remaining results of this paper, it is assumed that one is able to select a suitable $p \geq 1$ such that \eqref{mainObjective} is correctly specified. If it is not possible to identify such $p$, one may consider a sequence of models (indexed by $T$) where $p$ increases as more data become available. Eventually, if $(Y_t)_{t\geq 1}$ is an NLAR($p^\ast$) process for some $p^\ast \geq 1$, this will ensure that the model is correctly specified for large samples. Under suitable assumptions, Theorem~\ref{puttingThingsTogether} can in fact be adjusted to allow for such setting by adapting the ideas of \cite{wager2015adaptive} and keeping track of how constants depend on $p$. However, the resulting upper bound on the uniform deviation of regression trees from their partition-optimal counterparts seems to be rather sensitive to the value of $p$ and, thus, effectively demands that $p$ increases very slowly in $T$.
\end{remark}

\section{Concentration and consistency of forests}\label{consistency}
We start by translating the concentration inequality of Theorem~\ref{puttingThingsTogether} into the framework of random forests, which are constructed by averaging a number of trees. To this end, let $\mathcal{W}_k \coloneqq \{\boldsymbol{\Lambda} \subseteq \mathcal{V}_k\, :\, \vert \boldsymbol{\Lambda}\vert < \infty \}$ be the family of all finite collections of $k$-valid partitions. In line with \citet{wager2015adaptive}, given an element $\boldsymbol{\Lambda} =\{\Lambda_1,\dots, \Lambda_B\}$ of $\mathcal{W}_k$, the corresponding $k$-valid random forest $H_{\boldsymbol{\Lambda}}$ is defined as
\begin{equation}\label{forestEstimate}
H_{\boldsymbol{\Lambda}}(x) = \frac{1}{B}\sum_{b=1}^B T_{\Lambda_b} (x),\qquad x \in \mathbb{R}^p.
\end{equation}
The associated partition-optimal forest $H_{\boldsymbol{\Lambda}}^\ast$ is given by
\begin{equation*}
H_{\boldsymbol{\Lambda}}^\ast (x) = \frac{1}{B}\sum_{b=1}^B T^\ast_{\Lambda_b} (x),\qquad x\in \mathbb{R}^p.
\end{equation*}

\noindent As an immediate consequence of Theorem~\ref{puttingThingsTogether}, we obtain the following concentration inequality which applies uniformly across all $k$-valid forests (the result is stated without proof):

\begin{corollary}\label{RFconcentration}
	Suppose that \ref{noiseDist}--\ref{leavesAssumption} are satisfied. Then there exists a constant $\beta\in (0,\infty)$ such that
	\begin{equation*}
	\sup_{(x,\boldsymbol{\Lambda})\in \mathbb{R}^p\times \mathcal{W}_k} \vert H_{\boldsymbol{\Lambda}} (x) - H_{\boldsymbol{\Lambda}}^\ast (x) \vert \leq \beta \frac{(\log T)^2}{\sqrt{k}}
	\end{equation*}
	with probability at least $1-4T^{-1}$ for all sufficiently large $T$.
\end{corollary}

\noindent Note that all trees $T_{\Lambda_1},\dots, T_{\Lambda_B}$ in \eqref{forestEstimate} are based on the same data set $\mathcal{D}_T$ (the partitions $\Lambda_1,\dots \Lambda_B$ as well as the averages within the relevant leaves $A_{\Lambda_1}(x),\dots, A_{\Lambda_B}(x)$ depend on $\mathcal{D}_T$). In contrast, in the random forests of \citet{breiman2001random}, an initial bootstrap step is performed before growing each tree, meaning that trees are built on a bootstrap sample from $\mathcal{D}_T$ (with replacement) rather than on $\mathcal{D}_T$ itself. Once we have a concentration inequality as in Theorem~\ref{puttingThingsTogether} (or Corollary~\ref{RFconcentration}) at our disposal, it is not difficult to design trees in such a way that the corresponding random forests are consistent estimators of $f$. Roughly speaking, given that $f$ is smooth, and since each tree in a forest is close to its partition-optimal counterpart with high probability, it is sufficient to design the recursive partitioning scheme such that the maximal diameter of each leaf shrinks to zero as $T$ becomes large. Below we demonstrate how to refine the collection of $k$-valid partitions $\mathcal{V}_k$ in a suitable way and, subsequently, prove consistency of the corresponding forests. The construction will be similar to those of \cite{meinshausen2006quantile,wager2018estimation,wager2015adaptive}. We emphasize that the refinement considered here does not result in one particular random forest estimator; rather, a number of rules is outlined, and these will ensure consistency of any random forest estimator, which is built in line with them. With $\alpha \in (0,1/2)$, $k \geq 1$ and $m \geq 2k$, we call $\Lambda$ an $(\alpha,k,m)$-valid partition, and write $\Lambda\in \mathcal{V}_{\alpha,k,m}$, if it is recursive and obeys the following rules:
\begin{enumerate}[(i)]
	\item\label{maxLeaves} Any currently unsplit node with at least $m$ data points will eventually be split.
	\item\label{splitDirections} The probability $\rho_i = \rho_i (\mathcal{D}_T)$ that a given (feasible) node is split along the $i$-th direction is bounded from below for all $i=1,\dots, p$ by a strictly positive constant.
	\item\label{splitPoints} The split position is chosen such that each child node contains at least a fraction $\alpha \in (0,1/2)$ of the data points in its parent node.
	\item\label{minLeaves} All leaves of the tree contain at least $k$ data points.
\end{enumerate}
The corresponding $(\alpha,k,m)$-valid forest is given by \eqref{forestEstimate} with $\Lambda_1,\dots, \Lambda_B \in \mathcal{V}_{\alpha,k,m}$. Let us now briefly address the rules outlined in \ref{maxLeaves}--\ref{minLeaves}. Clearly, \ref{minLeaves} ensures $\mathcal{V}_{\alpha,k,m}\subseteq\mathcal{V}_k$, and thus $(\alpha,k,m)$-valid forests form a subclass of $k$-valid forests. Rule~\ref{maxLeaves} controls the maximal number of observations in each leaf of a tree, and $m=2k$ corresponds to a situation where one keeps splitting until placing another split would violate \ref{minLeaves}. In general, if $m$ is not too large relative to $T$, this condition ensures that the number of leaves becomes large and, hence, the partition becomes fine. Concerning \ref{splitDirections}, it ensures that, eventually, a split will be placed along any of the $p$ (canonical) directions of the input space $\mathbb{R}^p$. Such a condition makes sense for us when $p$ is thought of as being fixed and rather small, but will not be reasonable in sparse settings where $p\to \infty$, and one will instead design the algorithm in a way that detects important directions with high probability. On the other hand, $\rho_i$ is indeed allowed to depend on $\mathcal{D}_T$, so one may use the data to identify which of the directions that are most important and then, based on this, form the probabilities $\rho_1,\dots, \rho_p$. In a time series setting, it may be advantageous to favor splits along the first direction which corresponds to an observation that is likely to be highly dependent with the observed value of $Y$. Finally, \ref{splitPoints} is a balancing condition which prohibits\ ``edge splits''. This is a technical condition imposed to track the distribution of data points among leaves. In theoretical work on random forests within the regression setting, it is typical to impose assumptions similar to \ref{maxLeaves}--\ref{minLeaves}, see \cite{meinshausen2006quantile,wager2018estimation,wager2015adaptive}. On the other hand, standard implementations, such as the RandomForestRegressor from the \pkg{sklearn} library in Python and the \pkg{ranger} package in R, incorporate only \ref{maxLeaves}, \ref{splitDirections} and \ref{minLeaves}. 

Since consistency will be established by relying on Theorem~\ref{puttingThingsTogether}, we require that \ref{noiseDist}--\ref{leavesAssumption} are satisfied. Moreover, the following assumptions are imposed:
\begin{enumerate}[start=4,label=(A\arabic*)]
	\item\label{Lip} The function $f$ in \eqref{mainObjective} is $C$-Lipschitz, that is,
	\begin{equation*}
	\vert f(x^\prime)-f(x) \vert \leq C \lVert x^\prime - x\rVert\qquad \text{for all $x,x^\prime \in \mathbb{R}^p$}
	\end{equation*}
	with $C\in (0,\infty)$ being a suitable constant and $\lVert \: \cdot \: \rVert$ some norm on $\mathbb{R}^p$.
	
	\item\label{assumpMaxLeaves} It holds that $\log (T/m)/\log (\alpha^{-1}) \to \infty$ as $T \to \infty$.
\end{enumerate}

\noindent With assumptions \ref{noiseDist}--\ref{assumpMaxLeaves} in hand, we can now state the following consistency result for $(\alpha,k,m)$-forests applied to nonlinear autoregressive processes: 
\begin{theorem}\label{RFconsistency}
	Let $\hat{f}_T$ be an $(\alpha,k,m)$-forest and suppose that \ref{noiseDist}--\ref{assumpMaxLeaves} are satisfied. Then the following statements hold:
	\begin{enumerate}[(a)]
		\item\label{pointwise} $\hat{f}_T$ is a pointwise consistent estimator of $f$ in the sense that 
		\begin{equation*}
		\hat{f}_T(x) \longrightarrow f(x)\qquad \text{in probability as $T\to \infty$}.
		\end{equation*}
		for any $x\in \mathbb{R}^p$.
		\item\label{meanSquare} $\hat{f}_T(X)$ is a consistent estimator of the conditional mean $\mathbb{E}[Y \mid X]$ in the sense that
		\begin{equation*}
		\hat{f}_T(X) \longrightarrow \mathbb{E}[Y\mid X] \qquad \text{in probability as $T\to \infty$}.
		\end{equation*}
	\end{enumerate}
\end{theorem}

\begin{remark}\label{consistencyRemark}
	It should be emphasized that, since consistency is obtained through Theorem~\ref{puttingThingsTogether}, the averaging effect gained by considering \eqref{forestEstimate} rather than a single tree is not exploited in this setting. In particular, for the regression trees to concentrate around their partition-optimal counterparts, the number of observations in each leaf is required to approach infinity as $T$ becomes large (cf.\ \ref{leavesAssumption}). If this is not the case, averages within leaves do not converge, meaning that individual trees will be inconsistent estimators for $f$. In this case, consistency of $\hat{f}_T$ must be caused by improved accuracy gained by averaging trees.
\end{remark}

\section{A simulation study}\label{simulation}
In this section we consider a number of different specifications of $f$ in \eqref{mainObjective} and illustrate through simulations the results of Theorem~\ref{RFconsistency}. In all examples, the distribution of $\varepsilon_1$ is assumed to have a standard Laplace distribution, so that $h_\varepsilon (x) =\frac{1}{2} e^{-\vert x \vert}$ for $x\in \mathbb{R}$. As already mentioned, this choice meets the conditions imposed in \ref{noiseDist}. To keep things simple, we consider initially $p=1$ so that $f$ is one-dimensional and $(Y_t)_{t\geq 1}$ is a first order Markov chain. Within this setting, we choose four different specifications of $f$, namely
\begin{align}
\begin{aligned}\label{fSpecific}
\begin{array}{lll}
f(x) = 0.5 \sign (x)\min \{\vert x \vert , 10\}, &   & f(x) = -2x e^{-0.7x^2}+3x^2e^{-0.95x^2},\\
f(x) = \cos (5x)e^{-x^2} &  \text{and}  & f(x) =\min \{\vert x \vert, 0.75\} \min \{\vert x \vert,10 \}.
\end{array}
\end{aligned}
\end{align}
The first specification of $f$ satisfies $f(x) = 0.5x$ when $x\in [-10,10]$, and is constant outside of $[-10,10]$, and hence the corresponding process $(Y_t)_{t\geq 1}$ is intended to mimic the classical linear AR($1$) process. Indeed, it is very unlikely that $\vert Y_t \vert$ exceeds $10$, which means that there is only little practical difference between the two processes. The second specification is an example of an exponential AR model (see, e.g., \cite{an1996geometrical}), while the last two specifications of $f$ correspond to an oscillating function and a particular spline, respectively. In Figure~\ref{samplePaths}, we have simulated a sample path $Y_1,\dots, Y_{400}$ for each of these specifications of $f$. 

\begin{figure}
	\centering
	\begin{minipage}{0.485\textwidth}
		\begin{tikzpicture}
		\begin{axis}[
		axis lines=middle,
		axis on top = true,
		title={{\scriptsize $f(x) =0.5 \sign (x)\min \{\vert x \vert , 10\}$}},
		width=\linewidth, 
		xticklabel style={/pgf/number format/fixed},
		xmax=420,
		]
		\addplot[blue!50] table [col sep=comma,x=a,y=b,mark=none] {smallData2.csv};
		\end{axis}
		\end{tikzpicture}
	\end{minipage}
	\hfill
	\begin{minipage}{0.485\textwidth}
		\begin{tikzpicture}
		\begin{axis}[
		axis lines=middle, 
		axis on top = true, 
		title={{\scriptsize $f(x) = -2x e^{-0.7x^2}+3x^2e^{-0.95x^2}$}},
		width=\linewidth, 
		xticklabel style={/pgf/number format/fixed},
		xmax=420,
		]
		\addplot[blue!50] table [col sep=comma,x=a,y=c,mark=none] {smallData2.csv};
		\end{axis}
		\end{tikzpicture}
	\end{minipage}
	\\
	\begin{minipage}{0.485\textwidth}
		\begin{tikzpicture}
		\begin{axis}[axis lines=middle, 
		axis on top = true, 
		title={{\scriptsize $f(x) = \cos (5x)e^{-x^2}$}},
		width=\linewidth, 
		xticklabel style={/pgf/number format/fixed},
		xmax=420,
		]
		\addplot[blue!50] table [col sep=comma,x=a,y=d,mark=none] {smallData2.csv};
		\end{axis}
		\end{tikzpicture}
	\end{minipage}
	\hfill
	\begin{minipage}{0.485\textwidth}
		\begin{tikzpicture}
		\begin{axis}[
		axis lines=middle, 
		axis on top = true, 
		title={{\scriptsize $f(x) =\min \{\vert x \vert, 0.75\} \min \{\vert x \vert,10 \}$}},
		width=\linewidth, xticklabel style={/pgf/number format/fixed},
		xmax=420,
		]
		\addplot[blue!50] table [col sep=comma,x=a,y=e,mark=none] {smallData2.csv};
		\end{axis}
		\end{tikzpicture}
	\end{minipage}
	\caption{Simulations of $Y_1,\dots, Y_{400}$ from the model \eqref{mainObjective} for the four different specifications of $f$ considered in \eqref{fSpecific}.}\label{samplePaths}
\end{figure}
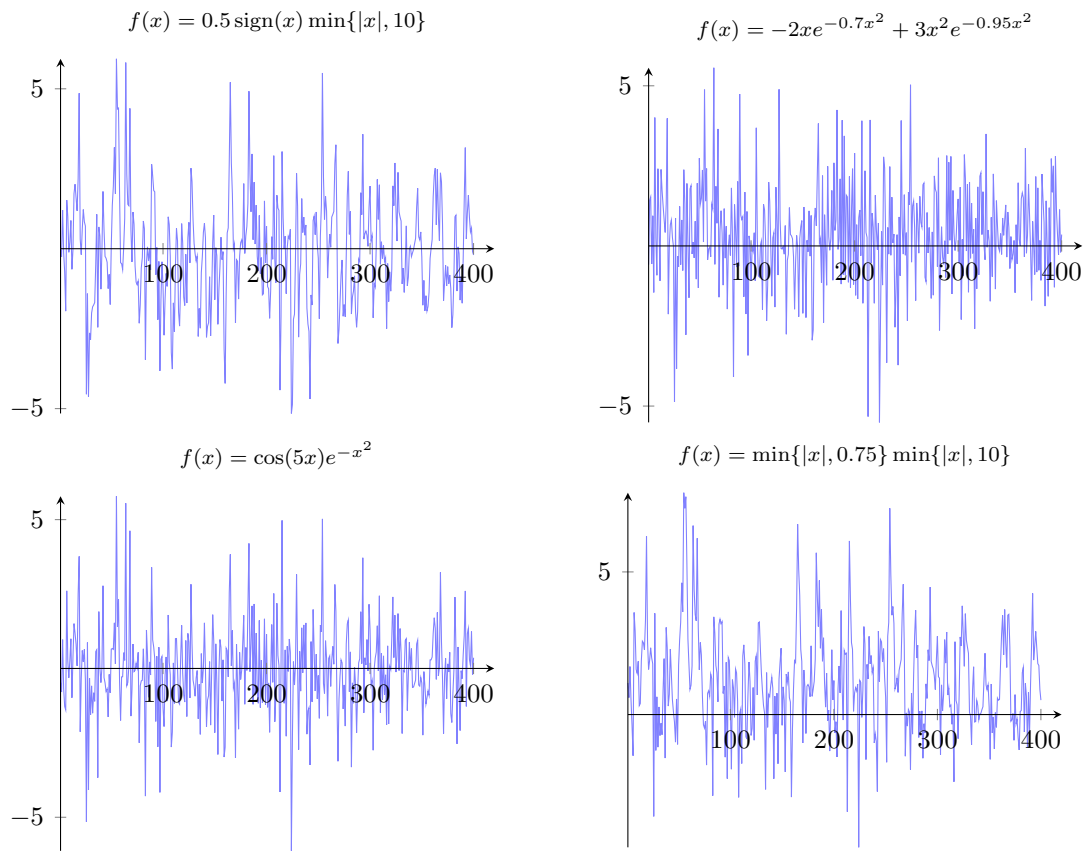

We consider estimation of $f$ by a random forest $\hat{f}_T$ across different sample sizes $T$ and we will be using the \pkg{ranger} package of R with $B = 500$ and $k = \lfloor 0.04(\log T)^4 \log\log T\rfloor$. To obtain diverse trees, we will use the extremely randomized trees of \citet{geurts2006extremely} which corresponds to setting the parameters $\texttt{replace} = \text{FALSE}$, $\texttt{sample.fraction} = 1$ and  $\texttt{splitrule} = \text{''extratrees''}$. Effectively, this means that split positions are chosen at random and that we build each tree using the entire sample $\mathcal{D}_T$ (no initial bootstrap step). Note that, while this implementation aligns with the $(\alpha, k,m)$-valid forests treated in Section~\ref{consistency}, $\alpha$ is not a prespecified parameter in the \pkg{ranger} package, yet in principle its value can be implicitly determined. In Figure~\ref{estimation} we compare $\hat{f}_T$ to $f$ on the interval $[-2,2]$ for each of the four different examples of $f$ presented in \eqref{fSpecific}.
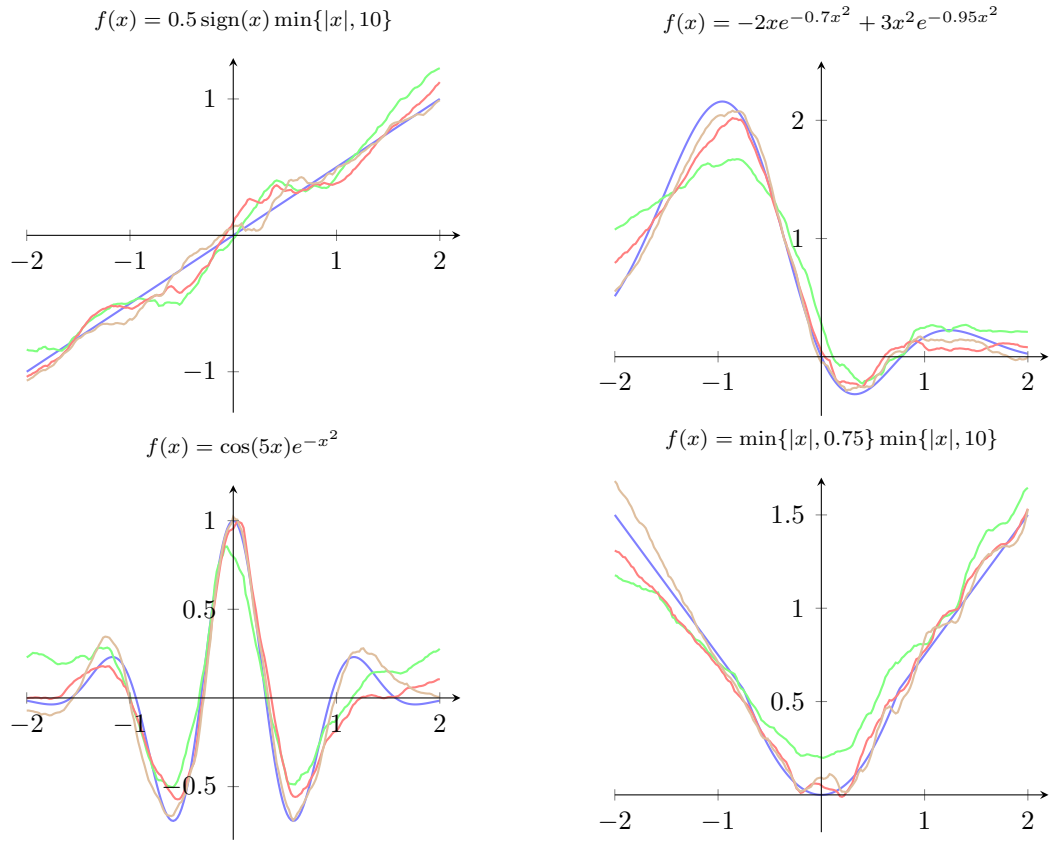
\begin{figure}
	\centering
	\begin{minipage}{0.485\textwidth}
		\begin{tikzpicture}
		\begin{axis}[axis lines=middle, 
		axis on top = true,
		axis line on top,
		title={{\scriptsize $f(x) = 0.5 \sign (x)\min \{\vert x \vert , 10\}$}},
		width=\linewidth, 
		xticklabel style={/pgf/number format/fixed},
		xmin=-2,
		xmax=2.2,
		ymin=-1.3,
		ymax=1.3,
		]
		\addplot[blue!50, thick] table [col sep=comma,x=a,y=b,mark=none] {functions.csv};
		\addplot[green!50, thick] table [col sep=comma,x=b,y=c,mark=none] {pred1.csv};
		\addplot[red!50, thick] table [col sep=comma,x=b,y=d,mark=none] {pred1.csv};
		\addplot[brown!50, thick] table [col sep=comma,x=b,y=e,mark=none] {pred1.csv};
		\end{axis}
		\end{tikzpicture}
	\end{minipage}
	\hfill
	\begin{minipage}{0.485\textwidth}
		\begin{tikzpicture}
		\begin{axis}[axis lines=middle,
		axis on top = true,
		axis line on top,
		title={{\scriptsize $f(x) = -2x e^{-0.7x^2}+3x^2e^{-0.95x^2}$}},
		width=\linewidth, 
		xticklabel style={/pgf/number format/fixed},
		xmin=-2,
		xmax=2.2,
		ymin=-0.5,
		ymax=2.5,
		]
		\addplot[blue!50, thick] table [col sep=comma,x=a,y=c,mark=none] {functions.csv};
		\addplot[green!50, thick] table [col sep=comma,x=b,y=c,mark=none] {pred2.csv};
		\addplot[red!50, thick] table [col sep=comma,x=b,y=d,mark=none] {pred2.csv};
		\addplot[brown!50, thick] table [col sep=comma,x=b,y=e,mark=none] {pred2.csv};
		\end{axis}
		\end{tikzpicture}
	\end{minipage}
	\\
	\begin{minipage}{0.485\textwidth}
		\begin{tikzpicture}
		\begin{axis}[axis lines=middle,
		axis on top = true,
		axis line on top,
		title={{\scriptsize $f(x) = \cos (5x)e^{-x^2}$}},
		width=\linewidth, 
		xticklabel style={/pgf/number format/fixed},
		xmin=-2,
		xmax=2.2,
		ymin=-0.8,
		ymax=1.2,
		]
		\addplot[blue!50, thick] table [col sep=comma,x=a,y=d,mark=none] {functions.csv};
		\addplot[green!50, thick] table [col sep=comma,x=b,y=c,mark=none] {pred3.csv};
		\addplot[red!50, thick] table [col sep=comma,x=b,y=d,mark=none] {pred3.csv};
		\addplot[brown!50, thick] table [col sep=comma,x=b,y=e,mark=none] {pred3.csv};
		\end{axis}
		\end{tikzpicture}
	\end{minipage}
	\hfill
	\begin{minipage}{0.485\textwidth}
		\begin{tikzpicture}
		\begin{axis}[axis lines=middle,
		axis on top = true,
		axis line on top,
		title={{\scriptsize $f(x) =\min \{\vert x \vert, 0.75\} \min \{\vert x \vert,10 \}$}},
		width=\linewidth, 
		xticklabel style={/pgf/number format/fixed},
		xmin=-2,
		xmax=2.2,
		ymin=-0.2,
		ymax=1.7,
		]
		\addplot[blue!50, thick] table [col sep=comma,x=a,y=e,mark=none] {functions.csv};
		\addplot[green!50, thick] table [col sep=comma,x=b,y=c,mark=none] {pred4.csv};
		\addplot[red!50, thick] table [col sep=comma,x=b,y=d,mark=none] {pred4.csv};
		\addplot[brown!50, thick] table [col sep=comma,x=b,y=e,mark=none] {pred4.csv};
		\end{axis}
		\end{tikzpicture}
	\end{minipage}
	\caption{The four specifications of $f$ considered in \eqref{fSpecific} (blue) and the corresponding random forest estimator $\hat{f}_T$ based on sample sizes of $T=400$ (green), $T=1600$ (red) and $T=6400$ (brown).}\label{estimation}
\end{figure}
While the plots indicate the consistency of the random forest estimator in these examples (as should be the case), observations are in fact rather noisy, and hence the performance of the random forest is indeed remarkable. To support this, Figure~\ref{scatters} shows scatter plots of the data $\mathcal{D}_T = \{(Y_0,Y_1),\dots, (Y_{T-1},Y_T)\}$ for $T=400$ and two specifications of $f$.
\begin{figure}
	\centering
	\begin{minipage}{0.485\textwidth}
		\begin{tikzpicture}
		\begin{axis}[
		axis lines=middle,
		title={{\scriptsize $f(x) = -2x e^{-0.7x^2}+3x^2e^{-0.95x^2}$}},
		axis on top = true,
		width=\linewidth, 
		xticklabel style={/pgf/number format/fixed},
		xmin=-3,
		xmax=3,
		ymin=-3,
		ymax=3,
		]
		\addplot[blue!50, thick] table [col sep=comma,x=a,y=b,mark=none] {functions2.csv};
		\addplot[gray!40, only marks, mark=*, mark size=0.7pt] table [col sep=comma,x=c,y=d,mark=*] {scatters2.csv};
		\end{axis}
		\end{tikzpicture}
	\end{minipage}
	\hfill
	\begin{minipage}{0.485\textwidth}
		\begin{tikzpicture}
		\begin{axis}[
		axis lines=middle,
		title={{\scriptsize $f(x) = \cos (5x)e^{-x^2}$}},
		axis on top = true,
		width=\linewidth, 
		xticklabel style={/pgf/number format/fixed},
		xmin=-2.5,
		xmax=2.5,
		ymin=-2.5,
		ymax=2.5,
		]
		\addplot[blue!50, thick] table [col sep=comma,x=a,y=c,mark=none] {functions2.csv};
		\addplot[gray!40, only marks, mark=*, mark size=0.7pt] table [col sep=comma,x=e,y=f,mark=*] {scatters2.csv};
		\end{axis}
		\end{tikzpicture}
	\end{minipage}
	\caption{Scatter plots of the data $\mathcal{D}_{400}$ under two of the specifications of $f$ considered in \eqref{fSpecific}.}\label{scatters}
\end{figure}
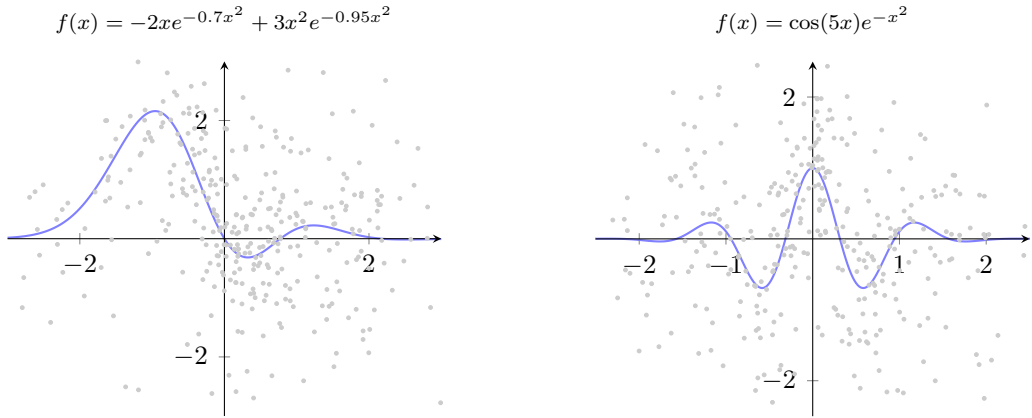
Furthermore, we note that choosing the parameter $k$ in finite samples is not a trivial task, and the choice used above is rather arbitrary (the assumption of \ref{leavesAssumption} concerns only its asymptotic behavior). Nevertheless, its value can have a significant impact on performance as it controls the bias--variance tradeoff of the estimator. While optimal tuning of $k$ is outside the scope of this paper, we illustrate its effect on $\hat{f}_T$ in Figure~\ref{leafEffect} where we estimate two of the functions in \eqref{fSpecific} for different values of $k$ using a sample of size $T = 1600$. For comparison, the value used for $k$ in Figure~\ref{estimation} when $T=1600$ was $236$.

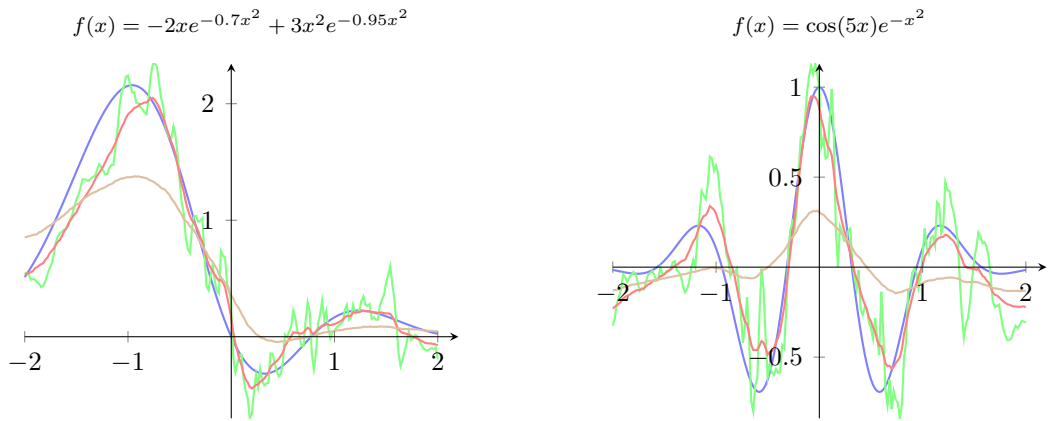
\begin{figure}
	\centering
	\begin{minipage}{0.485\textwidth}
		\begin{tikzpicture}
		\begin{axis}[
		axis lines=middle,
		axis on top = true,
		title={{\scriptsize $f(x) = -2x e^{-0.7x^2}+3x^2e^{-0.95x^2}$}},
		width=\linewidth, 
		xticklabel style={/pgf/number format/fixed},
		xmax=2.2,
		]
		\addplot[blue!50, thick] table [col sep=comma,x=a,y=c,mark=none] {functions.csv};
		\addplot[green!50, thick] table [col sep=comma,x=b,y=c,mark=none] {kSmooth1.csv};
		\addplot[red!50, thick] table [col sep=comma,x=b,y=d,mark=none] {kSmooth1.csv};
		\addplot[brown!50, thick] table [col sep=comma,x=b,y=e,mark=none] {kSmooth1.csv};
		\end{axis}
		\end{tikzpicture}
	\end{minipage}
	\hfill
	\begin{minipage}{0.485\textwidth}
		\begin{tikzpicture}
		\begin{axis}[
		axis lines=middle,
		axis on top = true,
		title={{\scriptsize $f(x) = \cos (5x)e^{-x^2}$}},
		width=\linewidth, 
		xticklabel style={/pgf/number format/fixed},
		xmax=2.2,
		]
		\addplot[blue!50, thick] table [col sep=comma,x=a,y=d,mark=none] {functions.csv};
		\addplot[green!50, thick] table [col sep=comma,x=b,y=c,mark=none] {kSmooth2.csv};
		\addplot[red!50, thick] table [col sep=comma,x=b,y=d,mark=none] {kSmooth2.csv};
		\addplot[brown!50, thick] table [col sep=comma,x=b,y=e,mark=none] {kSmooth2.csv};
		\end{axis}
		\end{tikzpicture}
	\end{minipage}
	\caption{Two of the specifications of $f$ considered in \eqref{fSpecific} (blue) and the corresponding random forest estimator $\hat{f}_{1600}$ with $k=40$ (green), $k=160$ (red) and $k=640$ (brown).}\label{leafEffect}
\end{figure}

We conclude this section by indicating consistency of random forests in a more challenging setting. In particular, we consider $p=2$ and the following choice of $f$:
\begin{equation}\label{twoDim}
f(x_1,x_2) = x_1e^{-0.6x_1^2}-2(x_1^2e^{-0.3x_1^2} +x_2e^{-0.7x_2^2})+3x_2^2e^{-0.95x_2^2}.
\end{equation}
We rely on the \pkg{ranger} package once again with the same specifications as were used to obtain Figure~\ref{estimation}, but we pass in the additional parameter $\texttt{split.select.weights}=(1/2,1/2)$ so that the probability of splitting along a given direction is the same for both directions (that is, $\rho_1 = \rho_2 = 1/2$). To evaluate its performance, we compute the mean squared error
\begin{equation}\label{MSE}
\MSE = \frac{1}{\vert \mathcal{X}\vert}\sum_{x \in \mathcal{X}} (\hat{f}_T(x)-f(x))^2
\end{equation}
over the grid $\mathcal{X}\coloneqq \{-2,-1.75,\dots, 1.75,2\}^2$ for different values of $T$. In Figure~\ref{MSEs}, the $\MSE$ is depicted as a function of $10^{-4}T$.

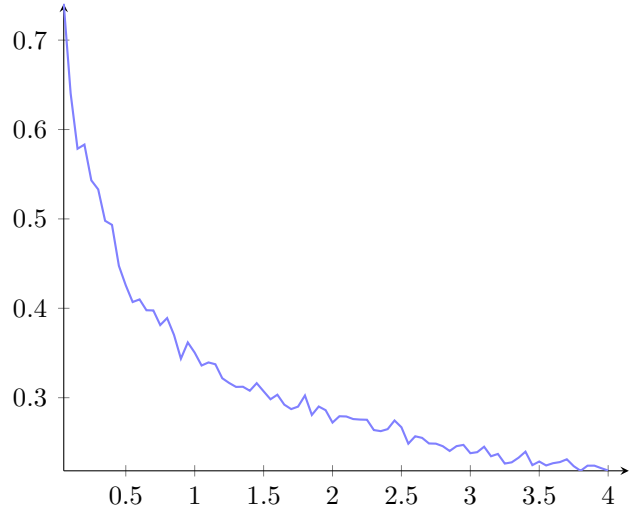
\begin{figure}
	\centering
	\begin{minipage}{0.6\textwidth}
		\begin{tikzpicture}
		\begin{axis}[
		axis lines=middle,
		width=\linewidth, 
		xticklabel style={/pgf/number format/fixed},
		xmax=4.15,
		]
		\addplot[blue!50, thick] table [col sep=comma,x=a,y=b,mark=none] {MSEs.csv};
		\end{axis}
		\end{tikzpicture}
	\end{minipage}
	\caption{The mean squared error \eqref{MSE} of the random forest estimator $\hat{f}_T$ as a function of $10^{-4}T$ when $f$ is given by \eqref{twoDim}.}\label{MSEs}
\end{figure}

\section{Proofs}\label{proofs}
It will be convenient to transform the input vector $X_t = (Y_{t-1},\dots, Y_{t-p})$ so that it takes values in $[0,1]^p$. Effectively, this can be done by applying a cumulative distribution function
\begin{equation}\label{CDF}
F_h(x) = \int_{-\infty}^x h(y)\, \dd y,\qquad x \in \mathbb{R}, 
\end{equation}
with $h\colon \mathbb{R}\to [0,\infty)$ being a probability density which is strictly positive almost everywhere. We extend the domain of $F_h$ to $\overline{\mathbb{R}} \coloneqq \mathbb{R}\cup \{\pm \infty\}$ by using the conventions $F_h(-\infty) = 0$ and $F_h(\infty) = 1$, so that the mapping
\begin{equation*}
\iota_h\colon (x_1,\dots, x_p)\longmapsto (F_h(x_1),\dots, F_h(x_p))
\end{equation*}
is one-to-one between $\smash{\overline{\mathbb{R}}}^p$ and $[0,1]^p$. The transformed input vector is defined by $Z_t = \iota_h(X_t)$. While there are no further restrictions on the choice of $h$, we will pick one that leads to good estimates on the density $h_Z$ of $Z_1$. 

\begin{lemma}\label{copula} Suppose that \ref{noiseDist} and \ref{sparsity} hold. Then there exists a constant $\zeta \in (1,\infty)$ and a probability density $h\colon \mathbb{R}\to [0,\infty)$, which is strictly positive almost everywhere, such that the density $h_Z\colon [0,1]^p \to [0,\infty)$ of $Z_1 = \iota_h(X_1)$ satisfies
	\begin{equation}\label{hZbound}
	\zeta^{-1} \leq h_Z(z)\leq \zeta
	\end{equation}
	for almost all $z\in [0,1]^p$.
\end{lemma}

\begin{proof}
	By \eqref{tailCond} in \ref{noiseDist} it holds that
	\begin{equation}\label{zetaBar}
	\bar{\zeta} \coloneqq \sup_{x\in \mathbb{R}}\frac{F_\varepsilon (x+M)}{F_\varepsilon (x-M)}\in (1,\infty).
	\end{equation}
	It follows as well from \ref{noiseDist} that $\varepsilon_1$ admits a density $h_\varepsilon$ which is strictly positive almost everywhere, and hence
	\begin{equation*}
	h(y) \coloneqq \frac{1 - \bar{\zeta}^{-1}}{\bar{\zeta} - \bar{\zeta}^{-1}}h_\varepsilon (y+M) + \frac{\bar{\zeta} -1}{\bar{\zeta} - \bar{\zeta}^{-1}}h_\varepsilon (y-M),\qquad y \in \mathbb{R},
	\end{equation*}
	is a valid density to use for defining $Z_t = \iota_h (X_t)$. To show that $h_Z$ meets \eqref{hZbound}, it suffices to establish that
	\begin{equation}\label{generatingSet}
	\zeta^{-1}\prod_{i=1}^p z_i \leq \mathbb{P}(F_h (Y_1)\leq z_1,\dots, F_h(Y_p)\leq z_p) \leq \zeta \prod_{i=1}^p z_i
	\end{equation}
	for all $z_1,\dots, z_p \in [0,1]$, where $F_h$ is the cumulative distribution function defined by \eqref{CDF} and $\zeta = \bar{\zeta}^p$. Since $\varepsilon_i -M \leq Y_i \leq \varepsilon_i + M$ by \ref{sparsity}, it follows immediately from the independence of $\varepsilon_1,\dots, \varepsilon_p$ and the monotonicity of $F_h$ that
	\begin{equation*}
	\prod_{i=1}^pF_\varepsilon (F_h^{-1}(z_i)-M)\leq \mathbb{P}(F_h (Y_1)\leq z_1,\dots, F_h(Y_p)\leq z_p)\leq \prod_{i=1}^pF_\varepsilon (F_h^{-1}(z_i)+M).
	\end{equation*}
	Consequently, we only need to show that both $F_\varepsilon (F_h^{-1}(z)+M)\leq \bar{\zeta} z$ and $F_\varepsilon (F_h^{-1}(z)-M)\geq \bar{\zeta}^{-1}z$ for an arbitrary $z\in [0,1]$. Observe that, by \eqref{zetaBar} and the definition of $h$, $\bar{\zeta}^{-1}F_\varepsilon (x+M)\leq F_h (x)\leq \bar{\zeta} F_\varepsilon (x-M)$ for all $x\in\overline{\mathbb{R}}$. In particular, by choosing $x = F_h (z)^{-1}$ we obtain
	\begin{equation*}
	\bar{\zeta}^{-1}F_\varepsilon (F_h^{-1}(z)+ M)\leq z  \leq \bar{\zeta} F_\varepsilon (F_h^{-1}(z) - M),
	\end{equation*}
	and this completes the proof.
	
\end{proof}

\noindent In all of the following, $Z_t=\iota_h (X_t)$ for some $h$ such that \eqref{hZbound} holds, and we will be using the notation $\# R \coloneqq \vert \{t\in \{1,\dots T\}\, :\, Z_t \in R\}\vert$, $\mu (R) \coloneqq \mathbb{P}(Z_1\in R)$, and $\eta (R)\coloneqq \mathbb{E}[Y_1\mid Z_1 \in R]$ for any given measurable set $R\subseteq [0,1]^p$. Note that if $\Lambda \in \mathcal{V}_{k}$, the partition $\bar{\Lambda}$ of $\smash{\overline{\mathbb{R}}}^p$ obtained by the exact same sequence of consecutive splits is again a $k$-valid partition and 
\begin{equation}\label{firstEq}
T_\Lambda (x) - T^\ast_\Lambda (x) =  T_{\bar{\Lambda}}(x) - T_{\bar{\Lambda}}^\ast (x)
\end{equation}
for all $x\in \mathbb{R}^p$. Moreover, for any $x\in \mathbb{R}^p$ we have that
\begin{equation}\label{secondEq}
T_{\bar{\Lambda}}(x) - T_{\bar{\Lambda}}^\ast (x) = \frac{1}{\# L }\sum_{t\colon Z_t\in L}Y_t - \eta (L)\eqqcolon G_T (L),
\end{equation}
where $L= \iota_h (A_{\bar{\Lambda}}(x))$. Here, and in what follows, it is implicitly understood that a sum of the form $\sum_{t\colon Z_t \in R}$ runs over $\{t\in \{1,\dots, T\}\, :\, Z_t \in R\}$. Since $A$ is a leaf of a $k$-valid partition of $\smash{\overline{\mathbb{R}}}^p$ with respect to $(X_1,\dots, X_T)$ if and only if $\iota_h(A)$ is a leaf of a $k$-valid partition of $[0,1]^p$ with respect to $(Z_1,\dots, Z_T)$, it follows from \eqref{firstEq} and \eqref{secondEq} that
\begin{equation}\label{empiricalReduction}
\sup_{L\in \mathcal{L}_{k}} \vert G_T (L)\vert = \sup_{(x,\Lambda)\in \mathbb{R}^p\times  \mathcal{V}_{k}}\vert T_\Lambda (x) - T^\ast_\Lambda (x)\vert,
\end{equation}
where $\mathcal{L}_{k}$ consists of all sets which are members of $k$-valid partitions of $[0,1]^p$. In particular, it suffices to prove uniform concentration inequalities for empirical averages over rectangles in $\mathcal{L}_{k}$. Still, there are infinitely many rectangles in $\mathcal{L}_{k}$, so one cannot simply analyze $\vert G_T(L)\vert$ and then rely on a union bound. We will follow the ideas of \citet{wager2015adaptive} who demonstrated that one only needs to understand the concentration over a much smaller set of approximating rectangles. In particular, we will make use of one of their results which states that there exists a rather small collection of rectangles in $[0,1]^p$ containing good approximations to any non-negligible rectangle in terms of Lebesgue measure. Since their result is more general than what is needed here (e.g., it can be used in situations where $p \to \infty$), we state a rather simplified version in Theorem~\ref{approxSets} below. To avoid introducing too many non-informative constants in the following, we introduce some convenient notation. For two sequences $(a_t)_{t\geq 1}$ and $(b_t)_{t\geq 1}$ we will write $a_t\lesssim b_t$ if there exists a constant $c\geq 1$ such that $a_t\leq cb_t$ for all $t$. If both $a_t \lesssim b_t$ and $b_t\lesssim a_t$ we write $a_t \asymp b_t$.

\begin{theorem}[\citet{wager2015adaptive}]\label{approxSets}
	Let $\varepsilon \asymp k^{-1/2}$ and $w \asymp k/ T$. Then there exists a collection of rectangles $\mathcal{R}_{\varepsilon , w}$ with the following two properties:
	\begin{enumerate}[(i)]
		\item For any rectangle $R\subseteq [0,1]^p$ with $\Leb (R)\geq w$, one can find $R_-,R_+ \in \mathcal{R}_{\varepsilon,w}$ satisfying
		\begin{equation}\label{approximatingSets}
		R_-\subseteq R \subseteq R_+ \quad \text{and}\quad e^{-\varepsilon}\Leb (R_+)\leq \Leb (R) \leq e^{\varepsilon} \Leb (R_-).
		\end{equation}
		\item\label{Rcardinality} The cardinality $\vert \mathcal{R}_{\varepsilon,w}\vert$ of $\mathcal{R}_{\varepsilon , w}$ satisfies the bound $\log \vert \mathcal{R}_{\varepsilon, w}\vert \lesssim \log T$.
	\end{enumerate}
\end{theorem}

\noindent Let $\varepsilon,w \in (0,1)$ be given as in Theorem~\ref{approxSets}. It follows that any given leaf $L\in \mathcal{L}_k^w \coloneqq \{L\in \mathcal{L}_{k}\, :\, \Leb (L) \geq w\}$ can be inner $\varepsilon$-approximated by a rectangle $L^\varepsilon_-$ from $\mathcal{R}_{\varepsilon,w}$ in the sense of \eqref{approximatingSets}. Moreover,
\begin{align}\label{decomp}
\begin{aligned}
\sup_{L\in \mathcal{L}^w_{k}} \vert G_T(L)\vert &\leq \sup_{L\in \mathcal{L}_{k}^w}\vert \eta ( L^\varepsilon_-) - \eta ( L)\vert + \sup_{L\in \mathcal{L}_{k}^w} \vert G_T(\smash{L^\varepsilon_-})\vert \\
&\quad + \sup_{L\in \mathcal{L}_{k}^w} \Bigl\vert\frac{1}{\# L}\sum_{t\colon Z_t\in L}Y_t - \frac{1}{\# L^\varepsilon_-}\sum_{t\colon Z_t\in L^\varepsilon_-}Y_t \Bigr\vert.
\end{aligned}
\end{align}
Thus, to obtain a concentration inequality for \eqref{empiricalReduction} it suffices to show that, for all large $T$ and with high probability, the three terms on the right-hand side of the inequality \eqref{decomp} are small and $\mathcal{L}_{k} = \mathcal{L}_{k}^w$. Bounding the first term of \eqref{decomp} is the easiest task.
\begin{lemma}\label{firstTerm} Suppose that \ref{noiseDist} and \ref{sparsity} are satisfied, and let $\varepsilon \asymp k^{-1/2}$ and $w \asymp k/ T$. Then
	\begin{equation*}
	\sup_{L\in \mathcal{L}_{k}^w}\vert \eta (L^\varepsilon_-)- \eta ( L)\vert \leq 2M\zeta^{2} \varepsilon,
	\end{equation*}
	where $\zeta\in (1,\infty)$ is given as in Lemma~\ref{copula}.
\end{lemma}
\begin{proof}
	By \ref{sparsity}, we find for an arbitrary leaf $L\in \mathcal{L}_{k}^w$ that
	\begin{align*}
	\vert \eta (L^\varepsilon_-)- \eta ( L)\vert &\leq \frac{1}{\mu (L)}\int_{L\setminus L^\varepsilon_-} \vert f(\iota_h^{-1}(z))\vert h_Z (z)\, \dd z\\
	&\quad + \frac{\mu (L) - \mu (L^\varepsilon_-)}{\mu (L) \mu (L^\varepsilon_-)}\int_{L_-^\varepsilon} \vert f(\iota_h^{-1}(z))\vert h_Z (z)\, \dd z \\
	&\leq 2M\frac{\mu (L)- \mu (L^\varepsilon_-)}{\mu (L)}.
	\end{align*}
	Moreover, Lemma~\ref{copula} and \eqref{approximatingSets} imply
	\begin{equation*}
	\mu (L)- \mu (L^\varepsilon_-) \leq \zeta (1-e^{-\varepsilon})\lambda (L) \leq \zeta^2 \varepsilon \mu (L),
	\end{equation*} 
	and this concludes the proof.
\end{proof}

\noindent The key to obtain estimates of the second and third term of \eqref{decomp}, as well as showing that $\mathcal{L}_{k} =\mathcal{L}^w_{k}$, with high probability is to establish good concentration inequalities for the counts $\# L$ and $\# \smash{L^\varepsilon_-}$, which apply across all $L\in \mathcal{L}_{k}^w$. As we will see in later proofs, by relying on Theorem~\ref{approxSets} and ideas similar to \cite[Theorem~10 and~Lemma~13]{wager2015adaptive}, it suffices to understand the concentration of $\# R$ across all rectangles in $\mathcal{R}_{\varepsilon,w}$ of non-negligible volume. This is the motivation for the following result, which relies on a Bernstein type inequality for weakly dependent processes.
\begin{lemma}\label{eventE}
	Suppose that \ref{noiseDist}--\ref{leavesAssumption} are satisfied, and let $\varepsilon \asymp k^{-1/2}$ and $w\asymp k/T$. Then there exists a constant $\gamma\in (0,\infty)$ such that
	\begin{equation}\label{s0}
	\sup_{R\in \mathcal{R}_{\varepsilon, w}\colon \mu (R)\geq w}\frac{\vert \# R - T \mu (R)\vert }{\sqrt{T \mu (R)}} \leq \gamma  \log T
	\end{equation}
	with probability at least $1-T^{-1}$ for all sufficiently large $T$.
\end{lemma}

\begin{proof}
	Note that, by a union bound, it suffices to establish that for any $R\in \mathcal{R}_{\varepsilon,w}$ with $\mu (R)\geq w$,
	\begin{equation}\label{s1}
	\mathbb{P}\biggl(\Bigl\vert \frac{\# R}{T}-\mu (R) \Bigr\vert > \gamma \log T\sqrt{\frac{\mu (R)}{T}}\biggr)\leq \frac{1}{\vert \mathcal{R}_{\varepsilon,w}\vert T}.
	\end{equation}
	To this end, observe that $(Y_t)_{t\geq 1}$ forms a stationary geometrically ergodic $p$-th order Markov chain (cf.\ \cite[Theorem~3.1]{an1996geometrical}). It is well-known that any such chain is exponentially $\alpha$-mixing (see, e.g., \cite[p.~89]{doukhan2012mixing}). In particular, the $t$-th $\alpha$-mixing coefficient $\alpha (t) \coloneqq \sup_{A\in \sigma (X_1),\, B\in \sigma (X_{t+1})} \vert \mathbb{P}(A\cap B) - \mathbb{P}(A)\mathbb{P}(B)\vert$ of $X_t = (Y_{t-1},\dots, Y_{t-p})$ satisfies
	\begin{equation}\label{alphaMixing}
	\log \alpha (t)  \lesssim - t,\qquad t \geq 1.
	\end{equation}
	Moreover, the $\alpha$-mixing coefficients of $(\mathds{1}_R(Z_t))_{t\geq 1}$ are obviously bounded by $(\alpha (t))_{t\geq 0}$ (which do not depend on $R$), and thus we can rely on a Bernstein type inequality for weakly dependent sequences \cite[Theorem~2]{merlevede2009bernstein} to establish that
	\begin{equation}\label{Bernstein}
	\log \mathbb{P}\Bigl(\Bigl\vert \frac{\# R}{T}-\mu (R) \Bigr\vert > x\Bigr) \lesssim
	- \frac{x^2 T}{\nu^2_{R}  + T^{-1} + x(\log T)^2},\qquad x>0,
	\end{equation}
	where 
	\begin{equation*}
	\nu^2_{R}\coloneqq \Var (\mathds{1}_R(Z_1)) + 2 \sum_{t=1}^\infty \vert \text{Cov} (\mathds{1}_R(Z_{t+1}),\mathds{1}_R(Z_1))\vert.
	\end{equation*}
	It is easy to see that $\vert \Cov (\mathds{1}_R(Z_{t+1}),\mathds{1}_R(Z_1))\vert\leq  \min \{\alpha(t) , \mu (R)\}$. From this inequality and the fact that $\sqrt{\alpha (t)} \leq \mu (R)$ as long as $t\gtrsim \log (T/k)$, which follows from \eqref{alphaMixing} and $\mu (R)\gtrsim k/T$, we deduce that
	\begin{equation}\label{varProxyBound}
	\nu_R^2 \lesssim 
	\mu (R)\Bigl(1+\log (T/k) + \sum_{t=1}^\infty \sqrt{\alpha (t)} \Bigr) 
	\lesssim \mu (R) \log T.
	\end{equation}
	By combining this variance bound with inequality \eqref{Bernstein} and using that $\mu (R) \gtrsim 1/T$ we get
	\begin{equation}\label{bernsteinWeakDep}
	\log \mathbb{P}\Bigl(\Bigl\vert \frac{\# R}{T}-\mu (R) \Bigr\vert > x\Bigr) \lesssim - 
	\frac{x^2T}{\max\{\mu (R)\log T , x (\log T)^2\}}.
	\end{equation}
	To put it differently, we may choose a sufficiently large constant $\bar{\gamma}$ such that for any fixed $\tau \in (0,\infty)$,
	\begin{equation}\label{probabilityBound}
	\mathbb{P}\Bigl(\Bigl\vert \frac{\# R}{T}-\mu (R) \Bigr\vert >x \Bigr) \leq \frac{1}{\tau}
	\end{equation}
	if
	\begin{equation}\label{condition}
	x\geq \bar{\gamma} \max\biggl\{\frac{(\log T)^2\log \tau}{T},\sqrt{\frac{\mu (R)}{T}\log T\log \tau} \biggr\}.
	\end{equation}
	Since $\mu(R)\gtrsim k/T$, the maximum of \eqref{condition} is equal to its last term if
	\begin{equation*}
	k\geq \kappa (\log T)^3\log \tau
	\end{equation*}
	for a suitable constant $\kappa$. Moreover, if $\tau = \vert \mathcal{R}_{\varepsilon,w}\vert T$, Theorem~\ref{approxSets}\ref{Rcardinality} shows that $\log \tau \lesssim \log T$, so if $k$ is chosen in accordance with \ref{leavesAssumption}, the last term of the maximum in \eqref{condition} is the dominating one when $T$ is large. Consequently, by choosing $x$ to be the right-hand side of \eqref{condition} with $\tau = \vert \mathcal{R}_{\varepsilon,w}\vert T$, we obtain
	\begin{equation*}
	\mathbb{P}\biggl(\Bigl\vert \frac{\# R}{T}-\mu (R) \Bigr\vert > \bar{\gamma} \sqrt{\frac{\mu (R)}{T}\log T (\log \vert \mathcal{R}_{\varepsilon, w}\vert + \log T)}\biggr)\leq \frac{1}{\vert \mathcal{R}_{\varepsilon,w}\vert T}.
	\end{equation*}
	By using Theorem~\ref{approxSets}\ref{Rcardinality} once again it follows that \eqref{s1} is satisfied for a suitable constant $\gamma$ and verifies that \eqref{s0} holds with probability at least $1-T^{-1}$ for all sufficiently large $T$.
\end{proof}

\noindent The next result shows how the inequality \eqref{s0} impacts the magnitude of the third term of \eqref{decomp}.

\begin{lemma}\label{secondTerm}
	Suppose that \ref{noiseDist}--\ref{leavesAssumption} are satisfied, and let $\varepsilon = k^{-1/2}$ and $w = k/(4\zeta T)$ where $\zeta\in (1,\infty)$ is given as in Lemma~\ref{copula}. Then, the inequality \eqref{s0} implies
	\begin{equation*}
	\sup_{L\in \mathcal{L}_k^w} \Bigl\vert\frac{1}{\# L}\sum_{t\colon Z_t\in L}Y_t - \frac{1}{\# \smash{L^\varepsilon_-}}\sum_{t\colon Z_t\in \smash{L^\varepsilon_-}}Y_t \Bigr\vert \leq 6(M+ \max_{t=1,\dots, T} \vert \varepsilon_t \vert)\frac{\zeta^2 + 2 \gamma \log T}{\sqrt{k}}
	\end{equation*}
	for all sufficiently large $T$, where $M$ is given by \eqref{finiteRegression}.
\end{lemma}

\begin{proof}
	First, observe that
	\begin{align}\label{termTwoBound}
	\begin{aligned}
	\MoveEqLeft\sup_{L\in \mathcal{L}_{k}^w} \Bigl\vert \frac{1}{\# L}\sum_{t\colon Z_t \in L}Y_t - \frac{1}{\# \smash{L^\varepsilon_-}}\sum_{t\colon Z_t \in \smash{L^\varepsilon_-}}Y_t \Bigr\vert \\
	&\leq 2\bigl(M+ \max_{t=1, \dots, T}\vert \varepsilon_t\vert \bigr) \sup_{L\in \mathcal{L}_{k}^w}\frac{\# L - \# \smash{L^\varepsilon_-}}{\# L}.
	\end{aligned}
	\end{align}
	It follows that we need to show how \eqref{s0} implicitly restricts $\# L$ and $\# \smash{L_-^\varepsilon}$. Initially, we will argue that \eqref{s0} implies
	\begin{align}\label{lemmaGoal}
	&\# L\ \ \leq  e^{\zeta^2 \varepsilon} T \mu (L)+ \gamma e^{\frac{1}{2}\zeta^2 \varepsilon} \log T \sqrt{T\mu (L)}, \\
	&\# L\ \ \geq \frac{T\mu (\smash{L^\varepsilon_-})-\gamma^2(\log T)^2}{2}, \label{lemmaGoal1} \\
	\text{and}\qquad &\# \smash{L^\varepsilon_-} \geq  T\mu (L^\varepsilon_-) - \gamma \log T \sqrt{T\mu (\smash{L^\varepsilon_-})} \label{lemmaGoal2}
	\end{align}
	for all $L\in \mathcal{L}_{k}^w$. Consider any rectangle $R\subseteq [0,1]^p$ with $\mu (R)\geq \zeta w$,  and note that such rectangle satisfies $\Leb (R) \geq w$ by Lemma~\ref{copula}. Consequently, Theorem~\ref{approxSets} implies the existence of an outer approximation $R_+\supseteq R$ from $\mathcal{R}_{\varepsilon,w}$ with $\Leb (R_+)\leq e^\varepsilon \Leb (R)$. The inequality \eqref{s0} shows in particular that
	\begin{equation}\label{outerAppIneq}
	\frac{\# R_+-T\mu (R_+)}{\sqrt{T\mu (R_+)}} \leq \gamma \log T.
	\end{equation}
	Obviously $\# R_+ \geq \# R$, and by Lemma~\ref{copula} the $\mu$-measure of $R_+$ is bounded in terms of that of $R$ as
	\begin{equation*}
	\mu (R_+)\leq \mu (R) + \zeta^2 (e^{\varepsilon}-1)\mu (R)\leq e^{\zeta^{2} \varepsilon}\mu (R).
	\end{equation*}
	By combining this with \eqref{outerAppIneq} we conclude that
	\begin{equation}\label{dd}
	\sup_{R\colon \mu (R)\geq \zeta w}\frac{\# R -e^{\zeta^{2} \varepsilon} T \mu (R)}{\sqrt{T \mu (R)}} \leq \gamma e^{\frac{1}{2}\zeta^{2} \varepsilon}\log T,
	\end{equation}
	where it is implicitly understood that the supremum only runs over rectangles in $[0,1]^p$. Now, if $R$ is a rectangle with $\mu (R)<2\zeta w$, we may expand it along one or more of the $p$ directions to obtain a new rectangle $\smash{\widetilde{R}}$ with $R\subseteq\smash{\widetilde{R}}\subseteq [0,1]^p$ and $\mu (\smash{\widetilde{R}})= 2\zeta w$. Thus, by \eqref{dd} this means that
	\begin{equation}\label{leafNumberBound}
	\# R \leq \Bigl(\frac{e^{\zeta^2\varepsilon}}{2} + \frac{\gamma e^{\frac{1}{2}\zeta^2 \varepsilon}\log T}{\sqrt{2k}}\Bigr)k.
	\end{equation}
	By \ref{leavesAssumption}, the last term in the parenthesis goes to zero and $e^{\zeta^2\varepsilon}$ goes to one as $T$ approaches infinity, so we establish that $\# R < k$ as long as $T$ exceeds a certain threshold (which does not depend on $R$). To put it differently, as long as $T$ is sufficiently large and for any rectangle $R\subseteq [0,1]^p$, the following implication holds:
	\begin{equation}\label{implicationLeafNumber}
	\# R \geq k\qquad \Longrightarrow \qquad \mu (R)\geq 2 \zeta w.
	\end{equation}
	Consider now any leaf $L\in \mathcal{L}_k$. By \eqref{implicationLeafNumber} it must be the case that $\mu (L)\geq 2 \zeta w$, and thus \eqref{lemmaGoal} is an immediate consequence of \eqref{dd}. Moreover, the $\mu$-measure of the inner $\varepsilon$-approximation $\smash{L^\varepsilon_-}$ of $L$ is bounded from below as
	\begin{equation}\label{innerMuBound}
	\mu (L^\varepsilon_-) \geq (1-\zeta^2 (1-e^{-\varepsilon}))\mu (L) \geq 2(1-\zeta^2 (1-e^{-\varepsilon})) \zeta w\geq \zeta w,
	\end{equation}
	where the last inequality applies as long as $T$ is large enough. Thus, \eqref{lemmaGoal2} is implied by \eqref{s0}. In order to prove \eqref{lemmaGoal1}, first note that
	\begin{equation}\label{p1}
	T\mu (L^\varepsilon_-) \leq \# L + \gamma \log T \sqrt{T\mu (\smash{L_{-}^\varepsilon})}
	\end{equation}
	by \eqref{lemmaGoal2}. By dividing both sides of \eqref{p1} with $\sqrt{T\mu (\smash{L^\varepsilon_-})}$ and using that $T\mu (\smash{L^\varepsilon_-})\geq k/4$ when $T$ is large (by \eqref{innerMuBound}) we obtain
	\begin{equation}\label{p2}
	\sqrt{T\mu (\smash{L^\varepsilon_-})} \leq \frac{2\, \# L}{\sqrt{k}}  + \gamma \log T.
	\end{equation}
	Now, by using the bound \eqref{p2} for the last term in \eqref{p1} and rearranging terms,
	\begin{equation*}
	\# L \geq \frac{T\mu (\smash{L^\varepsilon_-})- \gamma^2 (\log T)^2}{1+2\gamma \log T/\sqrt{k}}.
	\end{equation*}
	By \ref{leavesAssumption}, $2\gamma\log T /\sqrt{k}\leq 1$ when $T$ is sufficiently large, and this proves \eqref{lemmaGoal1}. Now we use \eqref{lemmaGoal}--\eqref{lemmaGoal2} to bound $(\# L - \# \smash{L_-^\varepsilon})/\# L $ uniformly across $L\in \mathcal{L}_k^w$. For an arbitrary leaf $L\in \mathcal{L}_k^w$ it follows by \eqref{lemmaGoal} that
	\begin{equation*}
	\bigl(e^{\frac{1}{2}\zeta^2 \varepsilon}\sqrt{T\mu (L)}+\gamma \log T\bigr)^2\geq \# L + \gamma^2 (\log T)^2,
	\end{equation*}
	and hence
	\begin{align}
	e^{\zeta^2\varepsilon}T\mu (L)&\geq \bigl(\sqrt{\# L + \smash{\gamma^2 (\log T)^2}}-\gamma \log T\bigr)^2 \notag \\
	&= \# L + 2\gamma^2 (\log T)^2 - 2 \gamma \log T \sqrt{\# L + \smash{\gamma^2 (\log T)^2}}\notag\\
	&\geq \# L - 4 \gamma \log T\sqrt{\# L},\label{t2}
	\end{align}
	where, due to \eqref{lemmaGoal1}, the last inequality applies as long as $T$ exceeds a certain threshold (which does not depend on $L$). Moreover, \eqref{lemmaGoal1} implies
	\begin{equation}\label{t1}
	\sqrt{T\mu (\smash{L^\varepsilon_-})}\leq \sqrt{2\, \# L} + \gamma \log T\leq 2\sqrt{\# L}
	\end{equation}
	and, as in \eqref{innerMuBound}, the $\mu$-measure of $\smash{L_-^\varepsilon}$ is bounded from below as
	\begin{equation}\label{t0}
	\mu (\smash{L^\varepsilon_-}) \geq (1- \zeta^2(1-e^{-\varepsilon}))\mu (L) \geq e^{-2\zeta^2\varepsilon}\mu (L).
	\end{equation}
	Both \eqref{t1} and \eqref{t0} require that $T$ is large. By starting from \eqref{lemmaGoal2}, and then using \eqref{t2}--\eqref{t0}, we get the estimate
	\begin{align*}
	\# \smash{L^\varepsilon_-} \geq e^{-2\zeta^2\varepsilon}T\mu (L) - 2 \gamma \log T \sqrt{\# L}\geq e^{-3\zeta^2\varepsilon} \# L - 6\gamma \log T \sqrt{\# L}
	\end{align*}
	for large $T$. Thus, for such $T$, 
	\begin{equation*}
	\sup_{L\in \mathcal{L}_{k}^w} \frac{\# L - \# \smash{L^\varepsilon_-}}{\# L} \leq 1-e^{-3\zeta^2 \varepsilon} + \sup_{L\in \mathcal{L}_{k}^w}\frac{6\gamma \log T}{\sqrt{\# L}}\leq \frac{3 \zeta^2 + 6 \gamma \log T}{\sqrt{k}}.
	\end{equation*}
	In view of \eqref{termTwoBound}, this finishes the proof.
\end{proof}

\begin{remark}\label{LclassRemark}
	Suppose that we are in the setting of Lemma~\ref{secondTerm}. In its proof it is in fact established that $\mathcal{L}_k=\mathcal{L}_k^w$ when \eqref{s0} holds and $T$ is large. For instance, this is an immediate consequence of \eqref{implicationLeafNumber}.
\end{remark}

In a similar way, we use \eqref{s0} to bound the second term of \eqref{decomp}; this is detailed in the following lemma.

\begin{lemma}\label{thirdTerm}
	Suppose that \ref{noiseDist}--\ref{leavesAssumption} are satisfied, and let $\varepsilon = k^{-1/2}$ and $w = k/(4\zeta T)$ where $\zeta \in (1,\infty)$ is given as in Lemma~\ref{copula}. Then, the inequality \eqref{s0} implies
	\begin{align*}
	\sup_{L\in \mathcal{L}_{k}^w} \vert G_T(\smash{L_-^\varepsilon})\vert &\leq  \frac{4M\gamma \log T}{\sqrt{k}} + 2\sup_{R\in \mathcal{R}_{\varepsilon,w}\colon \mu (R)\geq \zeta w}\frac{1}{T\mu (R)}\Bigl\vert \sum_{t\colon Z_t \in R}\varepsilon_t \Bigr\vert \\
	&\quad + 2\sup_{R\in \mathcal{R}_{\varepsilon,w}\colon \mu (R)\geq \zeta w} \frac{\bigl\vert\frac{1}{T}\sum_{t\colon Z_t \in R} f(X_t) - \mathbb{E}[f(X)\mathds{1}_R(Z)]\bigr\vert}{\mu (R)}
	\end{align*}
	for all sufficiently large $T$, where $M$ is given by \eqref{finiteRegression}.
\end{lemma}
\begin{proof}
	For any given rectangle $R$ we have the bound
	\begin{align}\label{thirdTermMainBound}
	\begin{aligned}
	\vert G_T(R)\vert &\leq M \frac{\vert \# R - T\mu (R)\vert}{\# R} + \frac{1}{\# R}\Bigl\vert \sum_{t\colon Z_t \in R} \varepsilon_t \Bigr\vert \\
	&\quad + \frac{T}{\# R}\Bigl\vert\frac{1}{T}\sum_{t\colon Z_t \in R}f(X_t) - \mathbb{E}[f(X)\mathds{1}_R(Z)]\Bigr\vert.
	\end{aligned}
	\end{align}
	When \eqref{s0} is satisfied and $T$ is large enough, it follows from \eqref{innerMuBound} (which holds under \ref{noiseDist}--\ref{leavesAssumption}) that
	\begin{equation}\label{thirdTerm0}
	\mathcal{R}^\prime \coloneqq \{R\in \mathcal{R}_{\varepsilon,w}\, :\, \mu (R)\geq \zeta w\}\supseteq \{L_-^\varepsilon \, :\, L\in \mathcal{L}_k^w\}.
	\end{equation}
	Moreover, for any $R\in \mathcal{R}^\prime$, \eqref{s0} implies immediately that
	\begin{equation}\label{thirdTerm1}
	\# R \geq T \mu (R)\Bigl(1-\frac{2\gamma \log T}{\sqrt{k}} \Bigr)\geq \frac{T\mu (R)}{2},
	\end{equation}
	and hence also that
	\begin{equation}\label{thirdTerm2}
	\frac{\vert \# R - T\mu (R)\vert }{\# R} \leq \frac{4\gamma \log T}{\sqrt{k}}
	\end{equation}
	as soon as $T$ exceeds a certain threshold (which is independent of $R$). By combining \eqref{thirdTermMainBound}--\eqref{thirdTerm2} we obtain the result.
\end{proof}

\noindent Since $\mathcal{R}_{\varepsilon ,w}$ is a rather small collection of sets, it is hinted by Lemmas~\ref{secondTerm} and~\ref{thirdTerm} that the only missing part in order to prove Theorem~\ref{puttingThingsTogether} is to obtain bounds on
\begin{equation*}
\max_{t=1,\dots, T}\vert \varepsilon_t\vert,\qquad \frac{1}{T} \Bigl\vert \sum_{t\colon Z_t \in R}\varepsilon_t \Bigr\vert\qquad \text{and}\qquad 
\Bigl\vert\frac{1}{T}\sum_{t\colon Z_t \in R} f(X_t) - \mathbb{E}[f(X)\mathds{1}_R(Z)]\Bigr\vert
\end{equation*}
for any $R\in \mathcal{R}_{\varepsilon, w}$ with $\mu (R)\geq \zeta w$. The first term is easy to handle, since it is a maximum of i.i.d.\ random variables satisfying Bernstein's condition \eqref{subGaussian}. The last two terms can be handled by relying on Bernstein type inequalities for martingale differences and weakly dependent random variables. We will go through the details below.

\begin{proof}[Proof of Theorem~\ref{puttingThingsTogether}]
	The proof goes by defining four events $\mathcal{E}_1$, $\mathcal{E}_2$, $\mathcal{E}_3$ and $\mathcal{E}_4$ and arguing that (i) the inequality \eqref{mainInequality} holds true on $\mathcal{E}_1\cap \mathcal{E}_2\cap \mathcal{E}_3\cap \mathcal{E}_4$, and (ii) each event $\mathcal{E}_i$ occurs with probability at least $1-T^{-1}$. With $\varepsilon = k^{-1/2}$ and $w = k/(4\zeta T)$, $\zeta \in (1,\infty)$ given as in Lemma~\ref{copula}, the events that we will consider are the following:
	\begin{align*}
	\mathcal{E}_1 &\coloneqq \Bigl\{\sup_{R\in \mathcal{R}_{\varepsilon,w}\colon \mu (R) \geq \zeta w} \frac{\vert \# R - T\mu (R)\vert}{\sqrt{T\mu (R)}} \leq \gamma \log T \Bigr\},\\
	\mathcal{E}_2 &\coloneqq \bigl\{ \max_{t=1,\dots, T} \vert \varepsilon_t\vert \leq c_1 \log T\bigr\}, \\
	\mathcal{E}_3 &\coloneqq \Bigl\{ \sup_{R\in \mathcal{R}_{\varepsilon,w}\colon \mu (R) \geq \zeta w} \frac{1}{T\mu (R)}\Bigl\vert \sum_{t\colon Z_t\in R} \varepsilon_t \Bigr\vert \leq c_2\frac{\log T}{\sqrt{k}} \Bigr\} \cup \mathcal{E}_1^c,\\
	\mathcal{E}_4 &\coloneqq \Bigl\{\sup_{R\in \mathcal{R}_{\varepsilon,w}\colon \mu (R)\geq \zeta w} \frac{\bigl\vert\frac{1}{T}\sum_{t\colon Z_t \in R} f(X_t) - \mathbb{E}[f(X)\mathds{1}_R(Z)]\bigr\vert}{\mu (R)} \leq  c_3\frac{\log T}{\sqrt{k}}\Bigr\}.
	\end{align*}
	Here $\gamma$ is the constant from Lemma~\ref{eventE}, while $c_1$, $c_2$ and $c_3$ will be introduced during the proof. Moreover, $\mathcal{E}_1^c$ refers to the complement of $\mathcal{E}_1$.
	
	\medskip
	
	\noindent \emph{Proof of (i)}: Suppose that the event $\mathcal{E}_1\cap \mathcal{E}_2\cap \mathcal{E}_3\cap \mathcal{E}_4$ has occurred. Then, by using \eqref{decomp} and Lemmas~\ref{firstTerm}, \ref{secondTerm} and \ref{thirdTerm}, it follows that
	\begin{align*}
	\sup_{L\in \mathcal{L}_k} \vert G_T (L)\vert &\leq \frac{2M\zeta^2}{\sqrt{k}} + \frac{4M\gamma \log T}{\sqrt{k}} + \frac{2 c_2\log T}{\sqrt{k}} + \frac{2c_3 \log T}{\sqrt{k}} \\
	&\quad + 6(M+c_1\log T)\frac{\zeta^2 + 2 \gamma \log T}{\sqrt{k}}.
	\end{align*}
	In view of this inequality, \eqref{empiricalReduction} and Remark~\ref{LclassRemark}, we conclude that \eqref{mainInequality} is satisfied on $\mathcal{E}_1\cap \mathcal{E}_2\cap \mathcal{E}_3\cap \mathcal{E}_4$ for a suitably chosen constant $\beta$.
	
	\medskip
	
	\noindent \emph{Proof of (ii)}: The content of Lemma~\ref{eventE} is exactly that $\mathbb{P}(\mathcal{E}_1)\geq 1- T^{-1}$. Since the moments of $\varepsilon_1$ meet \eqref{subGaussian}, its distribution is sub-exponential and \eqref{subexponential} holds. Moreover, $\varepsilon_1,\dots, \varepsilon_T$ are i.i.d.\ random variables, so by applying a union bound we obtain the estimate
	\begin{equation*}
	\mathbb{P}\bigl(\max_{t=1,\dots, T} \vert \varepsilon_t \vert >x\bigr) \leq \gamma_1 T e^{-\gamma_2 x},\qquad x >0.
	\end{equation*}
	In other words, $\max_{t=1,\dots, T} \vert \varepsilon_t\vert \leq x$ with probability at least $1-T^{-1}$ if $x\geq \log (\gamma_1 T^2)/\gamma_2$, and this shows $\mathbb{P}(\mathcal{E}_2)\geq 1- T^{-1}$ for some $c_1$.
	
	Next, consider any rectangle $R\in \mathcal{R}_{\varepsilon, w}$ with $\mu (R) \geq \zeta w$ and observe that $(\varepsilon_t \mathds{1}_R(Z_t))_{t\geq 1}$ is a martingale difference sequence with respect to the filtration $\mathcal{F}_t = \sigma (Y_s\, :\, s\leq t)$. Since $\varepsilon_t$ is independent of $\mathcal{F}_{t-1}$ and its moments satisfy \eqref{subGaussian},
	\begin{equation*}
	\mathbb{E}[\vert \varepsilon_t \mathds{1}_R(Z_t)\vert^m \mid \mathcal{F}_{t-1}] \leq m! c^{m-2}\mathds{1}_R(Z_t),\qquad m \geq 3.
	\end{equation*}
	In particular, these observations show that we can rely on a Bernstein (Freedman) type inequality for unbounded summands to obtain
	\begin{equation}\label{freedman}
	\log \mathbb{P}\Bigl(\frac{1}{T}\Bigl\vert\sum_{t\colon Z_t\in R} \varepsilon_t \Bigr\vert >x,\, \# R\leq y \Bigr)\lesssim - \frac{x^2T}{y/T+x}
	\end{equation}
	for any $x,y>0$. Such a result can, e.g., be found in \cite[Theorem~8.2.2]{victor1999decoupling}. Let $\gamma$ be the constant from Lemma~\ref{eventE}, consider specifically
	\begin{equation*}
	y= T\mu (R) + \gamma \log T \sqrt{T\mu (R)}
	\end{equation*}
	and note that $y\leq 2 T\mu (R)$ when $T$ is large (by \ref{leavesAssumption}). By using \eqref{freedman} with this choice of $y$ it follows that
	\begin{equation*}
	\log \mathbb{P}\Bigl(\Bigl\{\frac{1}{T}\Bigl\vert\sum_{t\colon Z_t\in R} \varepsilon_t \Bigr\vert >x\Bigr\} \cap \mathcal{E}_1 \Bigr)\lesssim -\frac{x^2T}{\max \{\mu (R),x\}}.
	\end{equation*}
	From this inequality we deduce the existence of a constant $\kappa$, such that for any $\tau >0$,
	\begin{equation}\label{mg1}
	\mathbb{P}\Bigl(\Bigl\{\frac{1}{T}\Bigl\vert\sum_{t\colon Z_t\in R} \varepsilon_t \Bigr\vert >x\Bigr\} \cap \mathcal{E}_1 \Bigr)\leq \frac{1}{\tau}
	\end{equation}
	when
	\begin{equation}\label{xThres}
	x  = \kappa \max \biggl\{\frac{\log \tau}{T},\sqrt{\frac{\mu (R) }{T} \log \tau} \biggr\}.
	\end{equation}
	The maximum in \eqref{xThres} is equal to its second term if
	\begin{equation}\label{kBound}
	k\geq 4 \log \tau.
	\end{equation}
	It follows from Theorem~\ref{approxSets}\ref{Rcardinality} and \ref{leavesAssumption} that \eqref{kBound} is satisfied if $\tau = \vert \mathcal{R}_{\varepsilon,w}\vert T$ and $T$ is large, and thus \eqref{mg1} and \eqref{xThres} show that
	\begin{equation*}
	\mathbb{P}\biggl(\biggl\{\frac{1}{T}\Bigl\vert\sum_{t\colon Z_t\in R} \varepsilon_t  \Bigr\vert >c_2 \log T\sqrt{\frac{\mu (R)}{4T}} \biggr\} \cap \mathcal{E}_1 \Bigr) \leq \frac{1}{\vert \mathcal{R}_{\varepsilon,w}\vert T},
	\end{equation*}
	for a suitable constant $c_2$. By rearranging terms and using that $T\mu (R)\geq k/4$ it follows that
	\begin{equation*}
	\mathbb{P}\biggl(\biggl\{\frac{1}{T\mu (R)}\Bigl\vert\sum_{t\colon Z_t\in R} \varepsilon_t  \Bigr\vert >c_2 \frac{\log T}{\sqrt{k}}\biggr\} \cap \mathcal{E}_1 \biggr) \leq \frac{1}{\vert \mathcal{R}_{\varepsilon,w}\vert T}.
	\end{equation*}
	Consequently, by relying on a union bound over all rectangles in $\mathcal{R}_{\varepsilon,w}$, we establish that $\mathbb{P}(\mathcal{E}_3)\geq 1- T^{-1}$.
	
	To show $\mathbb{P}(\mathcal{E}_4)\geq 1- T^{-1}$ we consider again an arbitrary rectangle $R\in \mathcal{R}_{\varepsilon, w}$ with $\mu (R) \geq \zeta w$. The sequence $(f(X_t)\mathds{1}_R(Z_t))_{t\geq 1}$ is bounded and $\alpha$-mixing, and its associated mixing coefficients is bounded by those of $(X_t)_{t\geq 1}$, which we will denote by $(\alpha (t))_{t\geq 1}$. In particular, by \eqref{alphaMixing} it follows that the mixing coefficients of $(f(X_t)\mathds{1}_R(Z_t))_{t\geq 1}$ are bounded by an exponentially decaying sequence of numbers with a decay rate which does not depend on $R$. Consequently, as in the proof of Lemma~\ref{eventE}, we can again rely on \cite[Theorem~2]{merlevede2009bernstein} to obtain that
	\begin{equation}\label{lastEvent}
	\log \mathbb{P}\Bigl(\Bigl\vert \frac{1}{T}\sum_{t\colon Z_t \in R}f(X_t) - \mathbb{E}[f(X)\mathds{1}_R(Z)] \Bigr\vert > x\Bigr) \lesssim -\frac{x^2 T}{\nu_R^2 + T^{-1}+x(\log T)^2}
	\end{equation}
	where
	\begin{equation*}
	\nu_R^2 \coloneqq \Var (f(X_1)\mathds{1}_R(Z_1)) + 2 \sum_{t=1}^\infty \vert \Cov (f(X_{t+1})\mathds{1}_R(Z_{t+1}),f(X_1)\mathds{1}_R(Z_1))\vert.
	\end{equation*}
	Note that 
	\begin{equation*}
	\inf \{y\in [0,\infty)\, :\, \mathbb{P}(\vert f(X)\vert \mathds{1}_R(Z)>y)\leq u\} \leq M\mathds{1}_{\{u\leq \mu (R)\}},
	\end{equation*}
	so it follows by Rio's covariance inequality \cite[Theorem~1.1]{rio1993covariance} that
	\begin{equation*}
	\vert \Cov (f(X_{t+1})\mathds{1}_R(Z_{t+1}),f(X_1)\mathds{1}_R(Z_1))\vert \leq 4M^2 \min \{\alpha (t), \mu (R)\}.
	\end{equation*}
	Thus, by using the same arguments as in the proof of Lemma~\ref{eventE} (in relation to \eqref{varProxyBound}) we establish $\nu_R^2 \lesssim \mu (R) \log T$, meaning that \eqref{lastEvent} implies
	\begin{align}\label{lastEvent2}
	\log \mathbb{P}\Bigl(\Bigl\vert \frac{1}{T}\sum_{t\colon Z_t \in R}f(X_t) - \mathbb{E}[f(X)\mathds{1}_R(Z)] \Bigr\vert > x\Bigr)
	\lesssim - \frac{x^2T}{\max\{\mu (R)\log T , x (\log T)^2\}}.
	\end{align}
	Since the right-hand side of \eqref{lastEvent2} is the same as in \eqref{bernsteinWeakDep}, we can use the exact same arguments to verify the existence of a constant $c_3$ such that
	\begin{equation*}
	\mathbb{P}\biggl(\Bigl\vert \frac{1}{T} \sum_{t\colon Z_t\in R}f(X_t) - \mathbb{E}[f(X)\mathds{1}_R(Z)]\Bigr\vert >c_3\log T\sqrt{\frac{\mu (R)}{4T}} \biggr) \leq \frac{1}{\vert \mathcal{R}_{\varepsilon,w}\vert T}
	\end{equation*}
	when $T$ exceeds a certain threshold (not depending on $R$). In particular, 
	\begin{equation*}
	\mathbb{P}\biggl(\frac{\bigl\vert\frac{1}{T}\sum_{t\colon Z_t \in R} f(X_t) - \mathbb{E}[f(X)\mathds{1}_R(Z)]\bigr\vert}{\mu (R)} >c_3\frac{\log T}{\sqrt{k}} \biggr) \leq \frac{1}{\vert \mathcal{R}_{\varepsilon,w}\vert T}
	\end{equation*}
	from which it follows by a union bound over rectangles in $\mathcal{R}_{\varepsilon,w}$ that $\mathbb{P}(\mathcal{E}_4)\geq 1- T^{-1}$. We have now argued that both (i) and (ii) outlined in the beginning of the proof hold true, and hence we obtain the desired result.
\end{proof}

\noindent We now turn to the task of proving Theorem~\ref{RFconsistency}. To do so, we will make use of an auxiliary result which is presented in Lemma~\ref{diameter} below. In this formulation, $\diam (A)\coloneqq \sup_{x,x'\in A} \lVert x^\prime - x\rVert$ is the diameter of $A\subseteq \mathbb{R}^p$. 

\begin{lemma}\label{diameter} Suppose \ref{noiseDist}, \ref{sparsity} and \ref{assumpMaxLeaves} are satisfied and that $\Lambda = \Lambda (\mathcal{D}_T,\Theta) \in \mathcal{V}_{\alpha,k,m}$ for all $T$. Then, for any $x\in \mathbb{R}^p$, $\diam (A_\Lambda (x))\to 0$ as $T\to \infty$ with probability one.
\end{lemma}
\begin{proof}
	Let us represent the rectangle $A_\Lambda (x)$ in $\Lambda$ containing $x\in \mathbb{R}^p$ as $A_\Lambda (x) = A_\Lambda^1(x)\times \cdots \times A_\Lambda^p(x)$. Then, it suffices to show that
	\begin{equation}\label{consistencyTask}
	\Leb(A_\Lambda^i(x)) \longrightarrow 0,\qquad T\to \infty,
	\end{equation}
	with probability one for $i=1,\dots, p$. To this end, imagine the tree illustrating how $\Lambda$ is obtained by the recursive partitioning scheme and consider the path that $x$ takes down the tree from its root to the leaf $A_\Lambda (x)$. Let $d$ denote the depth of the tree at $x$ (that is, $x$ traverses exactly $d-1$ nodes before it reaches $A_\Lambda (x)$), and let $A^l$ be the node containing $x$ at depth $l$. In particular, $(A^l)_l$ is a decreasing sequence of sets with $A^1 = \mathbb{R}^p$ and $A^{d} = A_\Lambda (x)$, and $A^l_j \neq A^{l+1}_j$ for exactly one $j$ (with the notation $A = A_1\times \cdots \times A_p$). We let $S^i_l = \mathds{1}_{A^l_i \neq A^{l+1}_i}$ indicate whether the node containing $x$ at depth $l$ will be split along the $i$-th direction, and $\tau^i_l = \min\{ j\in \{\tau^i_{l-1}+1,\dots, d-1\}\, :\, S^i_j = 1\}$ the depth at which $x$ will experience the $l$-th split along the $i$-th direction ($\tau^i_0 \equiv 0$ and, say, $\tau^i_l =\infty$ if the set is empty). For an illustration of these definitions, see Figure~\ref{bernoulliRVs}. By the construction of the tree (specifically, the rules \ref{maxLeaves} and \ref{splitPoints} outlined in Section~\ref{consistency}) it holds that $m \geq T \alpha^{d-1}$, and hence 
	\begin{equation}\label{lowerDepthBound}
	d \geq 1+\frac{\log (T/m)}{\log (\alpha^{-1})}.
	\end{equation}
	Recall also that the tree is constructed in such a way that there exists a strictly positive constant $\rho$ which is a lower bound for the probability $\rho_i$ of splitting along the $i$-th direction at any given node. Suppose for simplicity (but without loss of generality) that, in fact, $\rho_i = \rho$. Then, $(S^i_l)_{l\geq 1}$ is a sequence of i.i.d.\ Bernoulli random variables and thus, with probability one,
	\begin{equation*}
	\sum_{l=1}^n S^i_l \longrightarrow \infty, \qquad n \to\infty.
	\end{equation*}
	Since the right-hand side of \eqref{lowerDepthBound} tends to infinity by \ref{assumpMaxLeaves}, it follows that 
	\begin{equation}\label{infDepth}
	\vert \{l \in \{1,\dots, d-1\} \, :\, \tau^i_l <\infty\}\vert \longrightarrow \infty,\qquad T\to \infty,
	\end{equation}
	almost surely. Consider an arbitrary number $l\in \{1,\dots, d-1\}$ with $\tau^i_l<\infty$ and let $h$ be any fixed density which aligns with Lemma~\ref{copula}. Then
	\begin{align}
	\frac{\Leb (F_h(A^{\tau^i_l+1}_i))}{\Leb (F_h(A^{\tau^i_l}_i))} &= 1- \frac{\Leb(\iota_h(A^{\tau^i_l})\setminus \iota_h(A^{\tau^i_l+1}))}{\Leb (\iota_h(A^{\tau^i_l}))}\notag \\
	&\leq 1- \zeta^{-2}\frac{\mathbb{P}_\Lambda(X\in A^{\tau^i_l}\setminus A^{\tau^i_l +1})}{\mathbb{P}_\Lambda (X\in A^{\tau^i_l})}\notag\\
	&\leq 1- \zeta^{-2}(1- \mathbb{P}_\Lambda (X\in A^{\tau^i_l+1}\mid X \in A^{\tau^i_l} ) ) \label{smallIntervals}.
	\end{align}
	(Note that, for a given interval $A\subseteq \mathbb{R}$, $F_h (A)\subseteq [0,1]$ refers to the image of $A$ under $F_h$.) Since the tree is grown with respect to the rule \ref{splitPoints},
	\begin{equation*}
	\frac{\vert \{t\in \{1,\dots, T\}\, :\, X_t \in A^{\tau^i_l+1}\}\vert}{\vert \{t\in \{1,\dots, T\}\, :\, X_t \in A^{\tau^i_l}\}\vert} \leq 1-\alpha,
	\end{equation*}
	so by a Glivenko--Cantelli theorem for ergodic processes (e.g., \cite[Theorem~1]{adams2010uniform}) we establish that
	\begin{equation}\label{boundStructure}
	\limsup_{T\to \infty}\mathbb{P}_\Lambda (X\in A^{\tau^i_l+1}\mid X\in A^{\tau^i_l})\leq 1- \alpha
	\end{equation}
	with probability one. By combining \eqref{smallIntervals} and \eqref{boundStructure} it follows that we can fix $\delta \in (0,1)$ such that, with probability one,
	\begin{equation}\label{contraction}
	\frac{\Leb (F_h(A^{\tau^i_l+1}_i))}{\Leb (F_h(A^{\tau^i_l}_i) )} \leq  1- \delta
	\end{equation}
	for all sufficiently large $T$. By definition, $A_i^{\tau^i_{l-1}+1} = A_i^{\tau^i_l}$, so by repeated use of \eqref{contraction} we obtain, for any given $n \in \{1,\dots, d-1\}$ with $\tau^i_n <\infty$,
	\begin{equation}\label{ExtendedIneq}
	\Leb (F_h (A^i_\Lambda (x))) \leq \Leb (F_h (A_i^{\tau_n^i+1})) = \prod_{l=1}^n \frac{\Leb (F_h (A_i^{\tau_l^i+1}))}{\Leb (F_h (A_i^{\tau_{l-1}^i+1}))} \leq (1-\delta)^n
	\end{equation} 
	for all sufficiently large $T$ almost surely. Thus, from \eqref{infDepth} we deduce that\linebreak $\Leb(F_h( A^i_\Lambda (x)))\to 0$ almost surely as $T\to \infty$, and this completes the proof.	
\end{proof}

\begin{figure}
	\centering
	\begin{tikzpicture}
	\node (x1) at (0,1.5) {$x$};
	\node (l1) at ([yshift=-1.13cm]x1) {$A^1= \mathbb{R}^p\colon S^i_1=0$};
	\draw[thick, dashed, ->] (x1)--(l1);
	\node at ([xshift=-7cm]l1) {\textit{Depth 1}};
	\node[circle, inner sep=1.2pt,draw,fill=black] (n1) at (0,0) {};
	\node (n2) at (-3,-1.5) {};
	\node (n3) at (3,-1.5) {};
	\draw (n2)--(n1);
	\draw[very thick] (n1)--(n3);
	\node (l2) at ([yshift=-3mm]n2) {};
	\node at ([xshift=-4cm]l2) {\textit{Depth 2}};
	\node at ([xshift=-1mm,yshift=-3mm]n3) {$A^2\colon S^i_2=1$};
	\node (n4) at ([xshift=1mm,yshift=-7mm]n2) {};
	\node (n5) at ([xshift=-1.5cm,yshift=-2cm]n4) {};
	\node (n6) at ([xshift=1.5cm,yshift=-2cm]n4) {};
	\node[circle, inner sep=1.2pt,draw,fill=black] (n7) at ([xshift=-1mm,yshift=-7mm]n3) {};
	\node (n8) at ([xshift=-1.5cm,yshift=-2cm]n7) {};
	\node (n9) at ([xshift=1.5cm,yshift=-2cm]n7) {};
	\draw[very thick] (n8)--(n7);
	\draw (n7)--(n9);
	\node (l3) at ([xshift=1mm,yshift=-3mm]n5) {};
	\node at ([xshift=-2.6cm]l3) {\textit{Depth 3}};
	\node (l32) at ([xshift=1mm,yshift=-3mm]n8) {$A^3 \colon S^i_3=1$};
	\node[circle, inner sep=1.2pt,draw,fill=black] (n10) at ([xshift=1mm,yshift=-7mm]n8) {};
	\node (n11) at ([xshift=-1.5cm,yshift=-2cm]n10) {};
	\node (n12) at ([xshift=1.5cm,yshift=-2cm]n10) {};
	\draw[very thick] (n11)--(n10);
	\draw (n10)--(n12);
	\node (l4) at ([xshift=1mm,yshift=-3mm]n11) {$A^4 = A_\Lambda (x)$};
	\node at ([xshift=-7cm]l4) {\textit{Depth 4}};
	\end{tikzpicture}
	\caption{An illustration of a tree with a depth of $d=4$ at $x$. In this example $\tau^i_1 = 2$ and $\tau^i_2 = 3$.}\label{bernoulliRVs}
\end{figure}
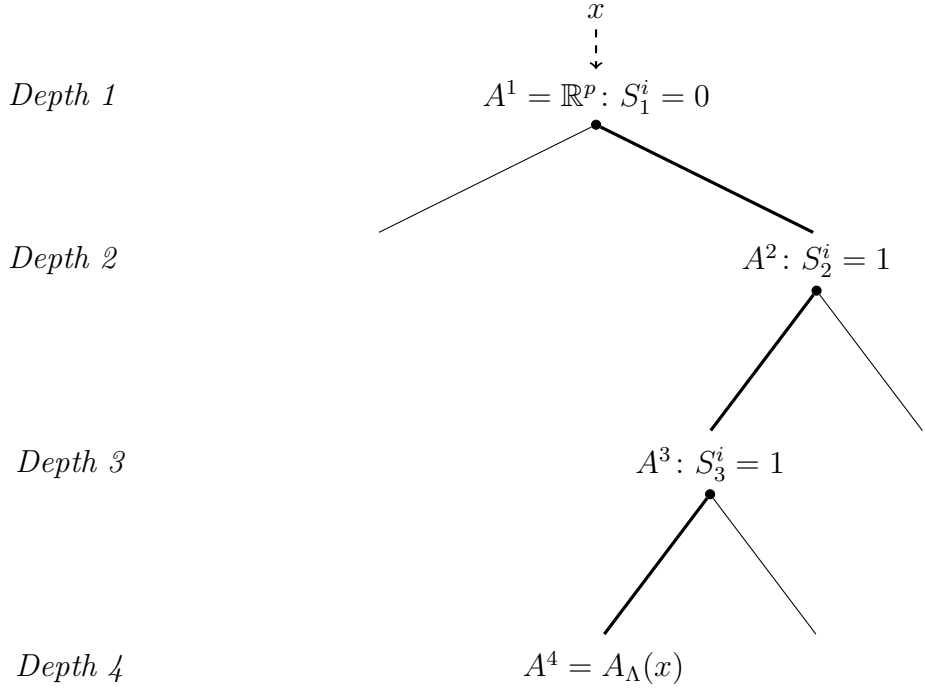

\begin{remark}\label{meinshausenRem}
	While the arguments used to prove Lemma~\ref{diameter} are somewhat similar to those of \cite[Lemma~2]{meinshausen2006quantile}, they may appear slightly more complicated. The reason is that the proof of \cite[Lemma~2]{meinshausen2006quantile} relies on an estimate of the form
	\begin{equation}\label{meinsRemark}
	\vert \{t\in \{1,\dots, T\}\, :\, X^i_t \in A^{\tau^i_n+1}_i\}\vert \leq T (1-\alpha)^n,
	\end{equation}
	from which one immediately deduces that $\Leb (F_h (A_\Lambda^i (x))) \leq \zeta(1-\alpha)^n$, which completes the proof. (Here $X^i_t = Y_{t-i}$ refers to the $i$-th entry of $X_t$.) However, the estimate \eqref{meinsRemark} does not apply in general under the rules \ref{maxLeaves}--\ref{minLeaves} (of Section~\ref{consistency}) upon which trees are grown, and hence a few additional arguments are needed to establish the alternative inequality \eqref{ExtendedIneq}.
\end{remark}

\begin{proof}[Proof of Theorem~\ref{RFconsistency}]
	By Corollary~\ref{RFconcentration},
	\begin{align*}
	\vert \hat{f}_T(x) - f(x) \vert &\leq \max_{b=1,\dots, B} \vert T^\ast_{\Lambda_b} (x)-f(x)\vert + o_p (1)	\\
	\text{and}\quad \vert \hat{f}_T(X) - f(X) \vert &\leq \max_{b=1,\dots, B} \vert T^\ast_{\Lambda_b} (X)-f(X)\vert + o_p (1)
	\end{align*}
	for suitable $\Lambda_1,\dots, \Lambda_B \in \mathcal{V}_{\alpha,k,m}$. Thus, it suffices to show that $T_\Lambda^\ast (x)\to f(x)$ and $T_\Lambda^\ast (X)\to f(X)$ in probability as $T\to \infty$ when $\Lambda \in \mathcal{V}_{\alpha, k,m}$ for all $T$. By using Lemma~\ref{diameter} together with the inequality
	\begin{equation*}
	\vert T^\ast_\Lambda (x) - f(x)\vert \leq \frac{\mathbb{E}_\Lambda[\vert f(X) - f(x)\vert \mathds{1}_{A_\Lambda (x)}(X)]}{\mathbb{P}_\Lambda (X\in A_\Lambda (x))} 
	\leq C \diam (A_\Lambda (x)),
	\end{equation*}
	which holds by \ref{Lip}, it follows that $T_\Lambda^\ast (x)\to f(x)$ almost surely and, in particular, in probability. Here, as in the proof of Lemma~\ref{diameter}, subscript $T$ indicates that we are conditioning on the randomness related to the partition $\Lambda$. The last part follows immediately from Tonelli's theorem as this implies that, on an event with probability one, $T^\ast_\Lambda (x)\to f(x)$ for (Lebesgue) almost all $x\in \mathbb{R}^p$. 
\end{proof}

\subsection*{Acknowledgements}
This work was supported by NSF grant DMS-2015379 for Davis and by Danish Council for Independent Research grant 9056-00011B for Nielsen.

\bibliographystyle{chicago}

\end{document}